\documentclass[UTF8,10pt,journal,cspaper,compsoc]{IEEEtran}
\usepackage{cite}

\ifCLASSINFOpdf
  \usepackage[pdftex]{graphicx}
  \DeclareGraphicsExtensions{.pdf,.jpeg,.png}
\else
\fi

\usepackage[cmex10]{amsmath}
\usepackage{amsfonts}
\usepackage[utf8]{inputenc}

\usepackage{amsthm}
\usepackage{algorithmicx}
\usepackage{algpseudocode}
\usepackage{booktabs}
\usepackage{multirow}
\usepackage[table]{xcolor}
\usepackage{url}

\newtheorem{definition}{Definition}

\newtheorem*{theorem*}{Theorem}

\algtext*{EndIf}
\algtext*{EndFor}
\algtext*{EndFunction}
\algtext*{EndWhile}

\begin{document}

\markboth{ACCEPTED TO APPEAR IN IEEE TRANSACTIONS ON SERVICES COMPUTING 2015, DOI 10.1109/TSC.2015.2480396}{}

\title{Hybrid Optimization Algorithm for Large-Scale QoS-Aware Service Composition}

\author{Pablo~Rodriguez-Mier,
	Manuel~Mucientes,
	and~Manuel~Lama% <-this % stops a space
\IEEEcompsocitemizethanks{
  \IEEEcompsocthanksitem P. Rodriguez-Mier, M. Mucientes and M. Lama
  work at the Centro de Investigación en Tecnoloxías da Información (CiTIUS), 
  Universidade de Santiago de Compostela, Spain.\protect\\
  E-mail: \{pablo.rodriguez.mier,manuel.mucientes,manuel.lama\}@usc.es
  }
  
\thanks{
  © 2015 IEEE. Personal use of this material is permitted. Permission
  from IEEE must be obtained for all other uses, in any current or future media, including
  reprinting/republishing this material for advertising or promotional purposes,
  creating new collective works, for resale or redistribution to servers or lists, or reuse of any
  copyrighted component of this work in other works.
}

}

\IEEEcompsoctitleabstractindextext{%

\begin{abstract}
In this paper we present a hybrid approach for automatic composition of Web services that generates semantic input-output based compositions with optimal end-to-end QoS, minimizing the number of services of the resulting composition. The proposed approach has four main steps: 1) generation of the composition graph for a request; 2) computation of the optimal composition 
that minimizes a single objective QoS function; 3) multi-step optimizations to reduce the search space by identifying equivalent and dominated services; and 4) hybrid local-global search to extract the optimal QoS with the minimum number of services. An extensive validation with the datasets of the Web Service Challenge 2009-2010 and randomly generated datasets shows that: 1) the combination of local and global optimization is a general and powerful technique to extract optimal compositions in diverse scenarios; and 2) the hybrid strategy performs better than the state-of-the-art, obtaining solutions with less services and optimal QoS.
\end{abstract}

}

\maketitle

\IEEEdisplaynotcompsoctitleabstractindextext

\IEEEpeerreviewmaketitle

\IEEEpeerreviewmaketitle

\section{Introduction}
\IEEEPARstart{W}{eb} services are self-describing software applications that can be published, discovered and invoked accross the Web using standard technologies\cite{alonso2004web}.
The functionality of a Web service is mainly determined by the functional properties that describe their behaviour in terms of its inputs, outputs, and also possibly additional descriptions that the services may have, such as preconditions and effects. These four characteristics, commonly abbreviated IOPEs, allow the composition and aggregation of Web services into composite Web services that achieve more complex functionalities and, therefore, solve complex user needs that cannot be satisfied with atomic Web services. However, compositions should go beyond achieving a concrete functionality and take into account other requirements such as Quality-of-Service (QoS) to generate also compositions that fit the needs of different contexts. The QoS determines the value of different quality properties 
of services such as response time (total time a service takes to respond to a request) or throughput (number of invocations supported in a given time interval), among others characteristics. These properties apply both to single services and to composite services, where each individual service in the composition contributes to the global QoS. For composite services this implies that having many different services with similar or identical functionality, but different QoS, may lead to a large amount of possible compositions that satisfy the same functionality with different QoS but also with a different number of services. 

However, the problem of generating automatic compositions that satisfy a given request with an optimal QoS is a very complex task, specially in large-scale environments, where many service providers offer services with similar functionality but with different QoS. This has motivated researchers to explore efficient strategies to generate QoS-aware Web service compositions from different perspectives \cite{rao2005survey,Strunk2010}. But despite the large number of strategies proposed so far, the problem of finding automatic compositions that minimize the number of services while guaranteeing the optimal end-to-end QoS is rarely considered. Instead, most of the work has focused on optimizing the global QoS of a composition or improving the execution time of the composition engines. An analysis of the literature shows that only a few works take into consideration the number of services of the resulting optimal QoS compositions. Some notable examples are \cite{Jiang2010a,Yan2009,Aiello2009,Chen2012}. Although 
most of these composition engines are quite efficient in terms of computation time, none of them are able to effectively minimize the total number of services of the solution while keeping the optimal QoS. 

The ability to provide not only optimal QoS but also an optimal number of services is specially important in large-scale scenarios, where the large number of services and the possible interactions among them may lead to a vast amount of possible solutions with different number of services but also with the same optimal QoS for a given problem. Moreover, there can be situations where certain QoS values are missing or cannot be measured. Although the prediction of QoS can partially alleviate this problem~\cite{ZhengMLK13}, it is not always possible to have historical data in order to build statistical models to accurately predict missing QoS. In this context, optimizing not only the available QoS but also the number of services of the composition may indirectly improve other missing properties. This has important benefits for brokers, customers and service providers. From the broker point of view, the generation of smaller compositions is interesting to achieve manageable compositions that are 
easier to execute, monitor, debug, deploy and scale. On the other hand, customers can also benefit from smaller compositions, specially when there are multiple solutions with the same optimal end-to-end QoS but different number of services. This is even more important when service providers do not offer fine-grained QoS metrics, since decreasing the number of services involved in the composition may indirectly improve other quality parameters such as communication overhead, risk of failure, connection latency, etc. This is also interesting from the perspective of service providers. For example, if the customer wants the cheapest composition, the solution with fewer services from the same provider may also require less resources for the same task.

However, one of the main difficulties when looking for optimal solutions is that it usually requires to explore the complete search space among all possible combinations of services, which is a hard combinatorial problem. In fact, finding the optimal composition with the minimum number of services is NP-Hard (see Appendix A).
Thus, achieving a reasonable trade-off between solution quality and execution time in large-scale environments is far from trivial, and hardly achievable without adequate optimizations.

In this paper we focus on the automatic generation of semantic input-output compositions, minimizing both a single QoS criterion and the total number of services subject to the optimal QoS.
The main contributions are:

\begin{itemize}
 \item A multi-step optimization pipeline based on the analysis of non-relevant, equivalent and dominated services in terms of interface functionality and QoS.
 \item A fast local search strategy that guarantees to obtain a near-optimal number of services while satisfying the optimal end-to-end QoS for an input-output based composition request.
 \item An optimal combinatorial search that can improve the solution obtained with the local search strategy by performing an exhaustive combinatorial search to select the composition with the minimum number of services for the optimal QoS.
\end{itemize}

We tested our proposal using the Web Service Challenge 2009-2010 datasets and, also, a different randomly generated dataset with a variable number of services. The rest of the paper is organized as follows: Sec. \ref{sec:Problem} introduces the composition problem, Sec. \ref{sec:Approach} describes the proposed approach, Sec. \ref{sec:Evaluation} presents the results obtained, and Sec. \ref{sec:Conclusions} gives some final remarks.

\section{Related Work}
Automatic composition of services is a fundamental and complex problem in the field of Service Oriented Computing, which has been approached from many different perspectives depending on what kinds of assumptions are made \cite{rao2005survey,Dustdar2005Survey,Bertoli2007,Strunk2010}. AI Planning techniques have been traditionally used in service composition to generate valid composition plans by mapping services to actions in the planning domain \cite{Ponnekanti2002,Bertoli2003,Sirin2004,Sirin2004workshop,Klusch2005,Akkiraju2006}. These techniques work under the assumption that services are complex operators that are well defined in terms of IOPEs, so the problem can be translated to a planning problem and solved using classical planning algorithms. Most of these approaches have been mainly focused on exploiting semantic techniques \cite{Sirin2004,Akkiraju2006,Hatzi2011} and developing heuristics \cite{Klusch2005,Akkiraju2006,OhLK07} to improve the performance of the planners. As a result, and partly given 
by the 
complexity of generating satisfiable plans in the planning domain, these approaches do not generate neither optimal plans (minimizing the number of actions) nor optimal QoS-aware compositions.

Other approaches have studied the QoS-aware composition problem from the perspective of Operation Research, providing interesting strategies for optimal selection of services and optimizing the global QoS of the composition subject to multiple QoS constraints. A common strategy is to reduce the composition problem to a combinatorial Knapsack-based problem, which is generally solved using constraint satisfaction algorithms (such as Integer Programming) \cite{zeng2004qos,Yu2005a,Berbner2006,Alrifai2009,ZouLCHXX14} or Evolutionary Algorithms \cite{CanforaPEV05,WadaSYO12}. Some relevant approaches are \cite{zeng2004qos,Alrifai2009}. In \cite{zeng2004qos} the authors present AgFlow, a QoS middleware for service composition. They analyze two different methods for QoS optimization, a local selection and a global selection strategy. The second strategy is able to optimize the global end-to-end QoS of the composition using a Integer Linear Programming method, which performs better than the suboptimal local selection 
strategy. Similarly, in \cite{Alrifai2009} the authors 
propose a 
hybrid QoS selection approach that combines a global optimization strategy with local selection for large-scale QoS composition. The assumption made by all these approaches is that there is only one composition workflow with a fixed set of abstract tasks, where each abstract task can be implemented by a concrete service. Both the composition workflow and the service candidates for each abstract task are assumed to be prefined beforehand, so these techniques are not able to produce compositions with variable size.

A different category of techniques are graph-based approaches that 1) generate the entire composition by selecting and combining relevant services and 2) optimize the global QoS of the composition. These techniques usually combine variants or new ideas inspired by different fields, such as AI Planning, Operations Research or Heuristic Search, in order to resolve more efficiently the automatic QoS composition, usually for a single QoS criterion. Some relevant approaches in this category are the top-3 winners of the Web Service Challenge (WSC) 2009-2010 \cite{Jiang2010a,Yan2009,Aiello2009}. Concretely, the winners of the WSC challenge \cite{Jiang2010a}, presented an approach that automatically discovers and composes services, optimizing the global QoS. This approach also includes an optimization phase to reduce the number of services of the solution. Although the proposed algorithm has in general good performance, as demonstrated in the WSC, it cannot guarantee to obtain optimal solutions in terms of number of 
services. The other participants of the WSC have also the same limitation.

A recent and interesting approach in this category has been recently presented by Jiang et al. \cite{Jiang2014}. In this paper, the authors analyze the problem of generating top \textit{K} query compositions by relaxing the optimality of the QoS in order to introduce service variability. However, the compositions are generated at the expense of worsening the optimal QoS, instead of looking first for all possible composition alternatives with the minimum number of services that guarantee the optimal QoS. 

Another interesting graph-based approach has been presented in \cite{Chen2012}. In this paper, the authors propose a service removal strategy that detects services that are redundant in terms of functionality and QoS. Results show that service removal techniques can be very effective to reduce the number of services before extracting the final composition, as anticipated by other similar approaches \cite{Barakat2011,Wagner2011,RodriguezMier2011}. However, some important limitations of this work are: 1) The QoS is not always optimal, since the graph  generated for the composition is not complete as it does not contain all the relations between services (it is acyclic) and 2) although the redundancy removal is an effective technique that can be used also to prune the search space, this strategy itself cannot provide optimal results in terms of number of services, and it should be combined with exhaustive search to improve the solutions obtained.

In summary, despite the large number of approaches for automatic QoS-aware service composition there is a lack of efficient techniques that are not only able to optimize the global end-to-end QoS, but also effectively minimize the number of services of the composition. This paper aims to provide an efficient graph-based approach that uses a hybrid local-global optimization algorithm in order to find optimal compositions both in terms of single QoS criteria and in terms of minimum number of services.

\section{Motivation}
\label{sec:Motivation}

\begin{figure*}
 \centering
 \includegraphics[width=\textwidth]{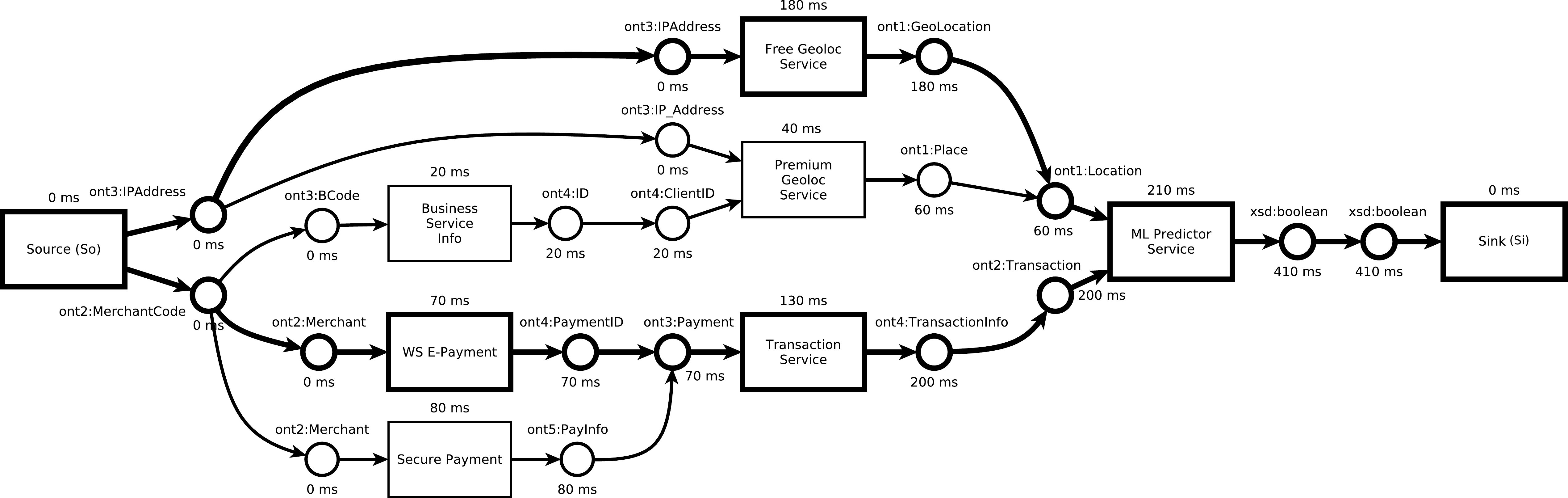}
 \caption{Example of a \textit{Service Match Graph} for a request $R=\{\{\textit{ont3:IPAddress},\textit{ont2:MerchantCode}\},\{\textit{xsd:boolean}\}\}$ to predict whether a business transaction is fraudulent or not. Each service is associated with an average response time. The optimal solution (\textit{Service Composition Graph}), with an overall response time of 410 ms and 4 services (excluding $So$ and $Si$) is highlighted.}
 \label{fig:RealExample}
\end{figure*}

The aim of the automatic service composition problem, as considered in this paper, is to automatically select the best combination of available QoS-aware services in a way that can fulfil a user request that otherwise could not be solved by just invoking a single, existing service. This request is specified in terms of the information that the user provides (inputs), and the information it expects to obtain (outputs). The resulting composition should meet this request with an optimal, single criterion end-to-end QoS and using as less services as possible. 

A motivating example of the problem is shown in Fig. \ref{fig:RealExample}. The figure represents a graph with all the relevant services for a request $R$ where the inputs are $\{\textit{ont3:IPAddress},\textit{ont2:MerchantCode}\}$ and the output is $\{\textit{xsd:boolean}\}$. The goal of this example is to obtain a composition to predict whether a business transaction is fraudulent or not. Each service (associated to a response time QoS) is represented by squares. Inputs and outputs are represented by circles. The graph also contains edges connecting outputs and inputs. These edges represent valid semantic \textit{matches} whenever an output of a service can be passed as an input of a different service. As can be seen, there are some inputs (\textit{ont1:Location},\textit{ont3:Payment}) that can be matched by more than one output, so there are many different ways to combine services to achieve the same goal. 

Although finding the proper combination of services in terms of their inputs/outputs is essential to generate a solution, it is not enough to obtain good compositions, since there can exist different combinations of services with different QoS. Moreover, many different combinations of services may produce compositions with a different number of services but the same end-to-end QoS.
For example, in Fig. \ref{fig:RealExample} we can select \textit{WS E-Payment} service or the \textit{Secure Payment} service to process the electronic payment. However, the second service has a higher response time. Using this leads to a sub-optimal end-to-end QoS of 420 ms. 
However, there are other situations where the selection of different services leads to compositions with different size but same end-to-end QoS. For example, both \textit{Free Geoloc Service} or the \textit{Premium Geoloc Service} can be selected to translate an \textit{IP} to a \textit{Location}. Although the second one has a better average response time (40 ms), it requires an additional service to obtain the \textit{ClientID} for verification purposes. However, selecting the \textit{Premium Geoloc Service} or the \textit{Free Geoloc Service} does not have an impact on the global QoS, since the \textit{ML Predictor Service} has to wait longer to obtain the \textit{Transaction} parameter (200 ms), but it has an impact on the total number of services of the solution.

The goal of this paper is to automatically generate, given a composition request, a graph like the one represented in Fig. \ref{fig:RealExample} as well as to extract the optimal end-to-end QoS composition with the minimum number of services from that graph.

\section{Problem Formulation}
\label{sec:Problem}

We herein formalize the main concepts and assumptions regarding the composition model used in our approach, which consists of a semantic, graph-centric representation of the service composition. These concepts are captured in three main models: 1) a service model, which is used to represent services and define how services can be connected or matched to generate composite services; 2) a graph-based composition model, which is used to represent both service interactions and compositions; and 3) a QoS computation model, which provides the operators required to compute the global QoS in a graph-based composition.

\subsection{Semantic Service Model}
\label{subsec:semantics}
The automatic composition of services requires a mechanism to select appropiated services based on their functional descriptions, as well as to automatic match the services together by linking their inputs and outputs to generate executable data-flow compositions. To this end, we introduce here the main concepts that we use in this paper to support the automatic generation of compositions. This model is an extension of a previous model used in \cite{RodriguezMier2015} to include QoS properties.

\begin{definition}
 A Composition Request $R$ is defined as a tuple $R=\{I_R,O_R\}$, where $I_R$ is the set of provided inputs, and $O_R$ the set of expected outputs. Each input and output is related to a semantic concept from the set $C$ of the concepts defined in an ontology $Ont$ ($In_w,Out_w \subseteq C$). We say that a composition satisfies the request $R$ if it can be invoked with the inputs in $I_R$ and returns the outputs in $O_R$.
\end{definition}

\begin{definition}
 A Semantic Web Service (hereafter ``service'') can be defined as a 
 tuple $w=\{In_w, Out_w, Q_w\} \in W$ where $In_w$ is a set of inputs 
 required to invoke $w$, $Out_w$ is the set of outputs returned by $w$ after its execution, $Q_w=\{q^1_w,\dots,q^n_w\}$ is the set of QoS values associated to the service,
 and $W$ is the set of all services available in the service registry. 
 
 Each input and output is related to a semantic concept from the set $C$ of the concepts defined in an ontology $Ont$ ($In_w,Out_w \subseteq C$).
 Each QoS value $q^i_w \in Q_w$ has a concrete type associated to a set of valid values $Q$. For example, the QoS values of a service $w$ with two different measures, an average response time of 20 ms and an average throughput of 1000 invocations/second, is represented as $Q_w=\{20ms, 1000~\textit{inv/s}\}$, where $20ms \in Q_{RT}$ and $1000~\textit{inv/s} \in Q_{TH}$.
\end{definition}

Semantic inputs and outputs are used to compose the functionality of multiple services by matching their inputs and outputs together. In order to measure the quality of the match, we need a matchmaking mechanism that exploits the semantic I/O information of the services. The different matchmaking degrees that are contemplated are \textit{exact}, \textit{plugin}, \textit{subsumes} and \textit{fail} \cite{Paolucci2002}.

\begin{definition}
\label{def:degree}
 Given $a,b \in C$, \textit{degree(a,b)} returns the degree of match between both concepts (exact, plugin, subsume or fail), which is determined by the logical relationship of both concepts within the Ontology.
\end{definition}

\begin{definition}
\label{def:match}
 Given $a,b \in C$, \textit{match(a,b)} holds if $degree(a,b) \neq fail$.
\end{definition}

In order to determine which concepts are matched by other concepts, we define a matchmaking operator ``$\otimes$'' that given two sets of concepts $C_1, C_2 \subseteq C$, it returns the concepts from $C_2$ matched by $C_1$.

\begin{definition}
\label{def:matchmaking}
 Given $C_1, C_2 \subseteq C$, we define ``$\otimes: C \times C \rightarrow C$'' such that
 $C_1 \otimes C_2 = \{c_2 \in C_2 | match(c_1, c_2), c_1 \in C_1\}$.
\end{definition}

We can use the previous operator to define the concepts of full and partial matching between concepts.

\begin{definition}
 Given $C_1, C_2 \subseteq C$, a full matching between $C_1$ and $C_2$ exists if $C_1 \otimes C_2 = C_2$, whereas
 a partial matching exists if $C_1 \otimes C_2 \subset C_2$.
\end{definition}

\begin{definition}
 Given a set of concepts $C' \subseteq C$, a service $w=\{In_w, Out_w\}$ is invokable if $C' \otimes In_w = In_w$, i.e., there is
 a full match between the provided set of concepts $C'$ and $In_w$, so the information required by $w$ is fully satisfied.
\end{definition}

This internal model used by the algorithm, which captures the core components required to perform semantic matchmaking and composition of services, is agnostic to how semantic services are represented. Thus, the algorithm is not bound to any concrete service description. Concretely, different service descriptions can be handled by the algorithm through the use of iServe importers for OWL-S, WSMO-lite, SAWSDL or MicroWSMO. For further details see \cite{Pedrinaci2010}.

\subsection{Graph-Based Composition Model}

In a nutshell, a data-flow composition of services can be seen as a set of services connected together through their inputs and output, using the semantic model defined before, in a way that every service in the composition is invocable and the invocation of each service in the composition can \textit{transform} a set of inputs into a set of outputs.
These concepts can be naturally captured by graphs, where the vertices represent inputs, outputs and services, and the edges represent semantic matches between inputs and outputs. Here we define the notion of \textit{Service Match Graph} and \textit{Service Composition Graph}. The \textit{Service Match Graph} is a graph that captures all the existent dependencies (matches) between all the relevant services for a composition request. The \textit{Service Composition Graph} is a particular case of the \textit{Service Match Graph} that represents a composition contained in the \textit{Service Match Graph}.

The \textit{Service Match Graph} represents the space of all possible valid solutions for a composition request $R$, and it is defined as a directed graph $G_S=(V,E)$, where:
\begin{itemize}
 \item $V = W_R \cup I \cup O \cup \{So, Si\}$ is the set of vertices of the graph, where $W_R \subseteq W$ is the set of relevant services, $I$ is the set of inputs and $O$ is the set of outputs. $Si$ and $So$ are two special services, called \textit{Source} and \textit{Sink} defined as $So=\{\emptyset, I_R\}$, $Si=\{O_R, \emptyset\}$.
 \item $E = IW \cup WO \cup OI$ is the set of edges in the graph where:
 \begin{itemize}
  \item $IW \subseteq \{(i_w,w) \mid i_w \in I \wedge w \in W \}$ is the set of input edges, i.e., edges
  connecting input concepts to their services.
  \item $WO \subseteq \{(w,o_w) \mid w \in W \wedge o_w \in O \}$ is the set of output edges, i.e., edges
  connecting services with their output concepts.
  \item $OI \subseteq \{(o_w,i_{w'}) \mid o_w,i_{w'} \in (I \cup O) \wedge match(o_w,i_{w'}) \}$ is the set of edges that represent a semantic match
  between an output of $w$ and an input of $w'$.
 \end{itemize}
\end{itemize}

There are also some restrictions in the edge set to ensure that each input/output belongs to a single service:

\begin{itemize}
 \item $\forall i \in I$ $d^{+}_{G_S}(i)=1 \wedge ch_{G_S}(i)=\{w\}, w \in W$ (each input has only one outgoing edge which connects the input with its service)
 \item $\forall o \in O$,$d^{-}_{G_S}(o)=1 \wedge par_{G_S}(o)=\{w\}, w \in W$ (each output has only one incoming edge which connects the output with its service)
 %\item $\forall i \in I, d^{-}_{G_S}(i) \geq 1 \wedge dest(i) \subseteq C$
\end{itemize}
Function $d^{+}_{G_S}(v)$ returns the outdegree of a vertex $v \in G_S$ (number of children vertices connected to $v$), whereas $d^{-}_{G_S}(v)$ returns the indegree of a vertex $v$ (number of parent vertices connected to $v$). The functions $ch_G(v)$ and $par_G(v)$ are the functions that returns the children vertices of $v$ and the parent vertices of $v \in G_S$, respectively.

Fig. \ref{fig:RealExample} shows an example of a \textit{Service Match Graph} where each service is associated with its average response time. As can be seen, this graph contains many different compositions since there are inputs in the graph that can be matched by the outputs of different services. For example, the parent nodes of the input \textit{ont1:Location} of the service \textit{ML Service Predictor} ($par_G(\textit{ont1:Location)}$) in Fig. \ref{fig:RealExample} are \textit{ont1:GeoLocation} and \textit{ont1:Place}, so the input is matched by two outputs $d^{-}_{G_S}(\textit{ont1:Location})=2$.

A \textit{Service Composition Graph}, denoted as $G_C=(V,E)$, represents a solution for the composition request where each input is exactly matched by one output. Formally, it is a subgraph of \textit{Service Match Graph} ($G_C \subseteq G_S$) that satisfies the following conditions:

\begin{itemize}
 \item $\forall i \in I, d^{-}_{G_C}(i)=1$ (each input is strictly matched by one output)
 \item $G_C$ is a Directed Acyclic Graph (DAG)
\end{itemize}
These conditions are important in order to guarantee that a solution is valid, i.e, each input is matched by an output of a service and each service is invocable (all inputs on the composition are matched with no cyclic dependencies). 
This definition of service composition is language-agnostic, so the resulting DAG is a representation of a solution for the composition problem which can be translated to a concrete language, such as OWL-S or BPEL.

\subsection{QoS Computation Model}

Before looking for optimal QoS service compositions, we need first to define a model to work with QoS over compositions of services which allow us to determine the best QoS that can be achieved for a given composition request on a service repository. When many services are chained together in a composition, the QoS of each individual service contributes to the global QoS of the composition. For example, suppose we want to measure the total response time of a simple composition with two services chained in sequence. The total response time is calculated as the sum of the response time of each service in the composition. However, if the composition has two services in parallel, the total time of the composition is given by the slowest services. Thus, the calculation of the QoS of a composition depends on the type of the QoS and on the structure of the composition.

In order to define the common rules to operate with QoS values in composite services, many approaches use a QoS computation model based on workflow patterns \cite{Cardoso2004}, which is adequate to measure the QoS of control-flow based compositions. However, this paper focuses on the automatic generation of optimal QoS-aware compositions driven by the data-flow analysis of the service dependencies (input-output matches) that are represented as a \textit{Service Match Graph}. 

In this section we explain the general graph-centric QoS computation model that we use, based on the \textit{path algebra} defined in \cite{carre1979graphs}. This model is better suited to compute QoS values in a \textit{Service Match Graph}, which, for extension, is also applicable to the particular case of the \textit{Service Composition Graph}.

\begin{definition}
 $(Q,\oplus,\ominus,\preceq)$ is a QoS algebraic structure to operate with a set of QoS values, denoted as $Q$. This set is equipped with the following elements:
 \begin{itemize}
  \item $\oplus: Q\times Q \rightarrow Q$ is a closed binary operation for aggregating QoS values
  \item $\ominus: Q\times Q \rightarrow Q$ is a binary operation for subtracting QoS values
  \item $\preceq$ is a total order relation on $Q$
 \end{itemize}
\end{definition}

This algebraic structure has the following properties:

\begin{enumerate}
  \item $Q$ is closed under $\oplus$ (any aggregation of two QoS values always returns a QoS value)
  \item The set $Q$ contains an identity element $e$ such that $\forall a \in Q, a \oplus e = e \oplus a = a$
  \item The set $Q$ contains a zero element $\phi$ such that $\forall a \in Q, \phi \oplus a = a \oplus \phi = \phi$
  \item The operator $\oplus$ is associative
  \item The operator $\oplus$ is monotone for $\preceq$ (preserves order). This implies that
 $\forall a,b,c \in Q, a \preceq b \Leftrightarrow a \oplus c \preceq b \oplus c$
 \item The operator $\ominus$ is the inverse of $\oplus$: $a \ominus b = c \Leftrightarrow a = c \oplus b$
 \end{enumerate}

Table \ref{table:qos} shows an example of the concrete elements in this algebra. Note that, for the sake of brevity, only the response time and throughput operators are represented in Table \ref{table:qos}. However, other QoS properties such as cost, availability, reputation, etc, can also be defined by instantiating the corresponding operators.
We denote $Q_{RT}$ the set of QoS values for response time (in milliseconds), $Q_{TH}$ the set of QoS values for throughput (invocations/second). The total order comparator $\preceq$ is required to be able to order and compare different QoS values. Given two QoS values $a,b\in Q$, $a \preceq b$ means that $a$ is equal or \textbf{\textit{better}} than $b$, whereas $b \preceq a$ means that $a$ is equal or \textbf{\textit{worse}} than $b$. The order depends on the concrete comparator defined on $Q$. For example, $Q_{RT}$ uses the comparator $\leq$ to order the response time, so $a,b \in Q_{RT}$, $a \preceq b \Leftrightarrow a \leq b$. For example, given two response times $10ms,20ms \in Q_{RT}$, $10ms \prec 20ms$ ($10ms$ is better than $20ms$) since $10ms < 20 ms$. However, $Q_{TH}$ uses the comparator $\geq$, so $a,b \in Q_{TH}$, $a \preceq b \Leftrightarrow a \geq b$. For example, 
given two throughput values $10\textit{~inv/s},20\textit{~inv/s} \in Q_{TH}$, $20\textit{~inv/s} \prec 10\textit{~inv/s}$ ($20\textit{~inv/s}$ is better than $10\textit{~inv/s}$) since $20\textit{~inv/s} > 10\textit{~inv/s}$. This order relation also affects the behavior of the $\min$ and $\max$ functions. The $\min$ function always selects the \textit{best} QoS value, whereas the $\max$ function always selects the \textit{worst} QoS value.

\begin{table}[htbp]
\centering
\caption{QoS algebra elements for response time and throughput}
\def\arraystretch{1.2}
\begin{tabular}{|c|c|c|c|c|c|}
\hline
\textbf{QoS ($Q$)} & \textbf{$a \oplus b$} & \textbf{$a \ominus b$} & \textbf{$e$} & \textbf{$\phi$} & \textbf{Order ($\preceq$)} \\ \hline \hline
$Q_{RT} =\mathbb{R}_{\geq 0} \cup \{\infty\}$ & a + b & a - b& 0 & $\infty$ & $\leq$ \\ \hline
$Q_{TH} = \mathbb{R}_{\geq 0} \cup \{\infty\}$ & min(a, b) & min(a,b) & $\infty$ & 0 & $\geq$ \\ \hline
\end{tabular}
\label{table:qos}
\end{table}

\begin{definition}
 $F_Q(w): W \rightarrow Q$ is a function that given a service $w \in W$, it returns its corresponding QoS value from $Q_w$ with type $Q$. This function can be seen as a function to measure the QoS of a service.
\end{definition}
 For example, in Fig. \ref{fig:RealExample}, $F_{Q_{RT}}(\textit{Trans. Service})=130 ms$.
\begin{definition}
 $V_Q(w): W \rightarrow Q$ is a function that given a service $w$, it returns its aggregated QoS value. This is defined as:
 \begin{eqnarray}
  V_Q(w) = \begin{cases}
    \max\limits_{\forall i \in In_w}(V^{in}_Q(i)) \oplus F_Q(w) &\text{if $In_w \neq \emptyset$}\\
     F_Q(w) &\text{if $In_w = \emptyset$}
  \end{cases}
  \label{eq:vq}
 \end{eqnarray}
\end{definition}

Informally, this function calculates the aggregated QoS of a service by taking the worst value of the QoS of its inputs plus the current QoS value of the service itself. Taking for example the service \textit{Premium Geoloc Service} from Fig. \ref{fig:RealExample}, $V_{Q_{RT}}(\textit{Premium Geoloc Service})$ is computed as $max(V^{in}_{Q_{RT}}(\textit{ont3:IP\_Address}), V^{in}_{Q_{RT}}(\textit{ont4:ClientID})) \oplus 40ms$, which is $max(0ms, 20ms) \oplus 40ms = 60ms$ (see Def. \ref{def:vqin}).

\begin{definition}
 $V^{out}_Q(o_w): O \rightarrow Q$ is a function that given an output of a service $w$, $o_w \in O$, it returns its aggregated QoS value. The aggregated QoS of an output is equal to the aggregated QoS of a service. Thus, it is defined as: 
 \begin{eqnarray}
  V^{out}_Q(o_w) = V_Q(w)
 \end{eqnarray}
\end{definition}

For example, the aggregated QoS of the output \textit{ont1:Place} ($V^{out}_{Q_{RT}}(\textit{ont1:Place})$) is equal to the aggregated QoS of its service \textit{Premium Geoloc Service} ($V_{Q_{RT}}(\textit{Premium Geoloc Service})$), which is equal to \textit{60ms}.

\begin{definition}
 $V^{in}_Q(i_w): I \rightarrow Q$ is a function that given an input of a service $w$, $i_w \in I$, it returns its optimal aggregated QoS value. This function is defined as:
  
  \begin{eqnarray}
  V^{in}_Q(i_w) = \begin{cases}
    \phi &\text{if $d^-_{G_S}(i_w) = 0$}\\
    V^{out}_Q(o_{w'}), o_{w'}\in par_G(i_w) &\text{if $d^-_{G_S}(i_w) = 1$}\\
    \min\limits_{\forall o_{w'} \in par_G(i_w)}(V^{out}_Q(o_{w'})) &\text{if $d^-_{G_S}(i_w) > 1$}\\
  \end{cases}
 \end{eqnarray}
 \label{def:vqin}
\end{definition}

Given an input $i_w \in In_w$ of a service $w$, this function returns the accumulated QoS for that input. If the evaluated input is not matched by any output ($d^-_{G_S}(i_w) = 0$), then the accumulated QoS of the input is undefined. If the evaluated input is matched by just one output ($d^-_{G_S}(i_w) = 1$), then its accumulated QoS value is equal to the accumulated QoS of that output. If the evaluated input can be matched by more than one output ($d^-_{G_S}(i_w) > 1$), i.e., there are many services that can match that input, then its accumulated QoS value is computed by selecting the optimal (best) QoS.

For example, the optimal aggregated QoS of the input \textit{ont3:Payment} from \textit{Transaction Service} ($V^{in}_{Q_{RT}}(\textit{ont3:Payment})$) is calculated as $\min(V^{out}_{Q_{RT}}(\textit{ont3:PaymentID}),V^{out}_{Q_{RT}}(\textit{ont5:PayInfo}))$ $=70ms$.

\begin{definition}
 We define $V^G_Q(g): G \rightarrow Q$ as a function that given a \textit{Service Match Graph} $g=(V,E)$, it returns its optimal aggregated QoS value. This is defined as:
  \begin{eqnarray}
  V^G_Q(g) = V_Q(Si), S_i \in V
 \end{eqnarray}
\end{definition}

Basically, the optimal QoS of a \textit{Service Match Graph} $G_S$ corresponds with the optimal aggregated QoS of its service $Si \in G_S$.

\subsection{Composition Problem}
Given a composition request $R=\{I_R, O_R\}$, a set of semantic services $W$, a semantic model and a QoS algebra, the composition problem considered in this paper consists of generating the \textit{Service Match Graph} $G_S$ and selecting a composition graph $G_C \subset G_S$ such that:
\begin{enumerate}
 \item $\forall G'_C, V^G_Q(G_C) \leq V^G_Q(G'_C)$, i.e., the composition graph has the best possible QoS
 \item $W_R \subseteq V,|W_R|$ is minimized (the composition graph contains the minimum number of services)
\end{enumerate}

\section{Composition Algorithm}
\label{sec:Approach}
On the basis of the formal definition of the automatic QoS-aware composition problem, in this section we present our hybrid approach strategy for automatic, large-scale composition of services with optimal QoS, minimizing the services involved in the composition. The approach works as follows: given a request, a directed graph with the relevant services for the request is generated. Once the graph is built, an optimal label-correcting forward search is performed in polynomial time in order to compute the global optimal QoS. This information is used later in a multi-step pruning phase to remove sub-optimal services. Finally, a hybrid local/global search is performed within a fixed time limit to extract the optimal solution from the graph. The local search returns a near-optimal solution fast whereas the global search performs an incremental search to extract the composition with the minimum number of services in the remaining time. In this section we explain each step of the algorithm, namely: 1) 
generation of the \textit{Service Match Graph}; 2) calculation of the optimal end-to-end QoS; 3) multi-step graph optimizations and 4) hybrid algorithm.

\subsection{Generation of the Service Match Graph}
Given a composition request, which specifies the inputs provided by the user as well as the outputs it expects to obtain, and a set of available services, the first step consists of locating all the relevant services that can be part of the final composition, as well as computing all possible matches between their inputs and outputs, according to the semantic model presented in Sec. \ref{subsec:semantics}. The output of this step is a \textit{Service Match Graph} that contains many possible valid compositions for the request, as the one represented in Fig. \ref{fig:GraphExample}. In a nutshell, the generation of the graph is calculated by selecting all invocable services layer by layer, starting with $So$ in the first layer (the source service whose outputs are the inputs of the request) and terminating with $Si$ in the last layer (the sink service whose inputs are the outputs of the request) \cite{Rodriguez-MierMVL12}.

\begin{figure}
\begin{algorithmic}[1]
\Function{ServiceMatchGraph}{$R=\{I_R, O_R\}, W$}    
  \State $C := I_R;\ W' := W; W_R := \{So,Si\}$
  \State $unmatchedIn := [\ ];\ availCon := I_R$
  \Repeat
      \State $W_{selected} = \emptyset$
      \State $W_{rel} := \{w \in W' \mid availCon \otimes In_w \neq \emptyset\}$  \label{alg:smg:relevant}
      \State $W_{rel} := W_{rel} \setminus W_R$ 
      \ForAll{$w_i =\{In_{w_i},Out_{w_i}\} \in W_{rel}$}
	  \State $U_{set} := unmatchedIn[w_i]$
	  \State $M_{set} := C \otimes U_{set}$
	  \State $unmatchedIn[w_i] := U_{set} \setminus M_{set}$ \label{alg:smg:unmatched}
	  \If{$M_{set} = \emptyset$}
	    \State $W_{selected} = W_{selected} \cup w_i$
	    \State $availCon := availCon \cup Out_{w_i}$
	  \EndIf
      \EndFor
      \State $W' := W' \setminus W_{selected}$
      \State $W_R := W_R \cup W_{selected}$
      \State $C := C \cup availCon$
      \State $availCon := \emptyset$
  \Until{$W_{selected}=\emptyset$}
  \State
  \Return COMPUTE-GRAPH($W_R$)
\EndFunction
\end{algorithmic}
\caption{Algorithm for generatig a Service Match Graph from a composition request $R$ and a set of services $W$.}\label{Alg:ServiceMatchGraph}
\end{figure}

The pseudocode of the algorithm is shown in Fig. \ref{Alg:ServiceMatchGraph}. The algorithm runs in polynomial time, selecting $W_{selected} \subseteq W$ services at each step. At each layer, the algorithm finds a potential set of relevant services whose inputs are matched by some outputs generated in the previous layer using the $\otimes$ operator (L.\ref{alg:smg:relevant}). Then, for each potential eligible service, the algorithm checks whether the service is invokable or not (i.e., all its inputs are matched by outputs of previous layers) by checking if all the unmatched inputs of the service are matches. All the inputs that are matched are removed from the unmatched set of inputs for the current service (L.\ref{alg:smg:unmatched}). If the service is invokable (has no unmatched inputs), it is selected and its outputs are added to the set of the available concepts. In case the service still has some unmatched inputs, these inputs are stored in a map to check it again in the next 
layer. For example, the first eligible services for the request shown in Fig. \ref{fig:GraphExample} are the services in the layer $L1$, which correspond with the services whose inputs are fully matched by $I_R$ (the set of output concepts produced in $L0$). The second eligible services are those services (placed in $L2$) whose inputs are fully matched by the outputs of the previous layers, and so on. The algorithm stops when no more services are added to the set of selected services. Finally, \textit{COMPUTE-GRAPH} computes all possible matches between the outputs and the inputs of the selected services. The output of this process is a complete \textit{Service Match Graph} that can contain cycles, as the one depicted in Fig. \ref{fig:GraphExample}.

\begin{figure}
 \centering
 \includegraphics[width=\columnwidth]{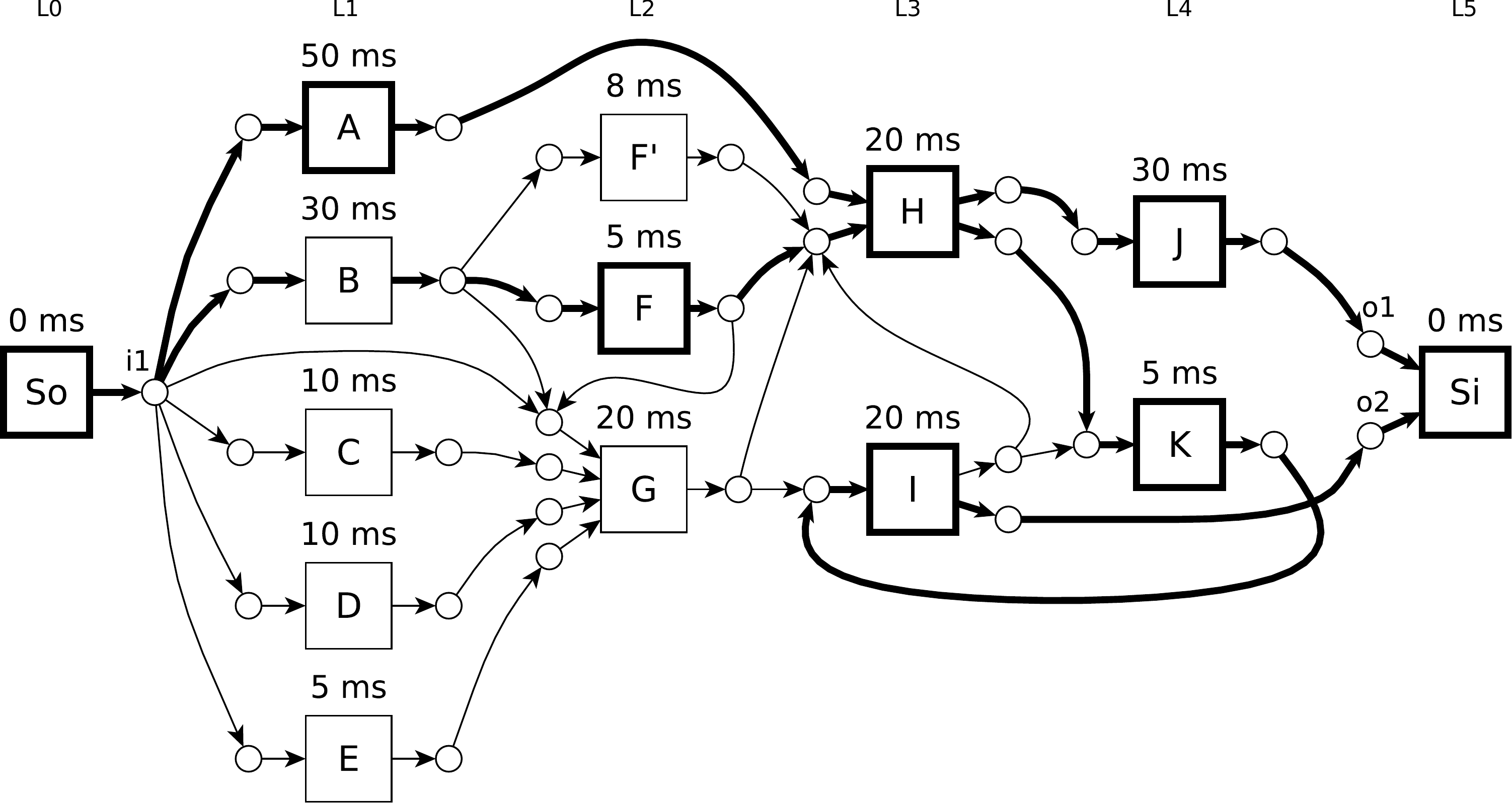}
 \caption{Graph example with the solution with optimal QoS and minimum number of services highlighted.}
 \label{fig:GraphExample}
\end{figure}

\subsection{Optimal end-to-end QoS}
Once the \textit{Service Match Graph} is computed for a composition request, the next step is to calculate the best end-to-end QoS achievable in the \textit{Service Match Graph}.
The optimal end-to-end QoS can be computed in polynomial time using a shortest path algorithm to calculate the best aggregated QoS values for each input and output of the graph, i.e., the best QoS values at which the outputs can be generated and the inputs are matched. In order to compute the optimal QoS, we use a generalized Dijkstra-based label-setting algorithm computed forwards from $So$ to $Si$ \cite{Rodriguez-Mier2012}, based on the algebraic model of the QoS presented in Sec. \ref{sec:Problem}. The optimality of the algorithm is guaranteed as long as the function defined to aggregate the QoS values ($\otimes$) is monotonic, in order to satisfy the principle of optimality. A proof can be found in \cite{Sobrinho2002}.

\begin{figure}
\begin{algorithmic}[1]
\Function{QoS-Update}{$G_S=\{V,E\}$}
  \State \textit{/*qos is a table indexed by inputs (i)}
  \State \textit{associated to their aggregated QoS (q)*/}
  \State $qos[i,q] \gets []$
  \ForAll{$i \in I, I \subset V$}
    \State $qos[i] \gets \phi$
  \EndFor
  \State $queue \gets So$
  \While{$queue \neq \emptyset$}
      \State \textit{/* Queue sorted by aggregated QoS */}
      \State $w \gets$ POP($queue$)
      \State $updated = \{\}$
      \ForAll{$o_w \in Out_w$}
	  \ForAll{$i_{w'} \in ch_G(o_w)$}
	    \If{$V_Q(w) \prec qos[i_{w'}]$}
	      \State $qos[i_{w'}] \gets V_Q(w)$
	      \State $updated \gets updated \cup w'$
	    \EndIf
	  \EndFor
      \EndFor
      \ForAll{$w \in updated$}
	\If{cost $w$ has been improved}
	  \State $queue \gets $INSERT$(w,queue)$
	\EndIf
      \EndFor
  \EndWhile
  \State
  \Return $qos$
\EndFunction
\end{algorithmic}
\caption{Dijkstra-based algorithm to compute the best QoS for each input and output in the \textit{Service Match Graph} $G_S$.}
\label{alg:qos}
\end{figure}

Fig. \ref{alg:qos} shows the pseudocode of the generalized Dijkstra-based label-setting algorithm. The algorithm starts assigning infinite QoS cost to each input in the graph in the table $qos$. An infinite cost for an input means that the input is still not \textit{resolved}. The first service to be processed is $So$. Each time a service $w$ is processed from the queue, the best accumulated QoS cost of each input $i_{w'}$ matched by the outputs of the service $w$ is recalculated. If there is an improvement (i.e., a match with a better QoS is discovered) the affected service is stored in $updated$ to recompute its new aggregated QoS. Finally, for each service $w \in updated$, we recompute its aggregated QoS using the updated values of each affected input. If the QoS has been improved, the service is added to the queue to expand it later.

\subsection{Graph optimizations}

Finding the composition with the minimum number of services is a very hard combinatorial problem which, in most cases, has a very large search space, mainly determined by the size of the \textit{Service Match Graph}. In order to improve the scalability with the number of services, we apply a set of admissible optimizations to reduce the search space. At each pass, the algorithm analyzes different criteria to identify services that are redundant or can be substituted by better ones, so the size of the graph decreases monotonically. The different passes that are sequentially applied are: 1) elimination of services that do not contribute to the outputs of the request; 2) pruning of services that lead to suboptimal QoS; 3) combination of interface (inputs/outputs) and QoS equivalent services; and 4) replacement of interface and QoS dominated services. These optimizations are an extension of the optimizations presented in \cite{RodriguezMier2015} to support QoS.

The \textbf{first pass} selects the set of reachable services in the \textit{Service Match Graph}. Starting from the inputs of $Si$, it selects all those services whose outputs match any inputs of $Si$. This step is repeated with the new services until the empty set is selected. Those services that were not selected do not contribute to the expected outputs of the composition and can be safely removed from the graph.

The \textbf{second pass} prunes the services of the graph that are suboptimal in terms of QoS, i.e., they cannot be part of any optimal QoS composition. To do so, we compute the maximum admissible QoS bound for each input in the graph. In a nutshell, the maximum bound of the inputs of a service $w$ can be calculated by selecting the maximum QoS bound among the bounds of all inputs matched by the outputs of the service $w$ and subtracting the QoS of $w$. This can be recursively defined as:
 
 \begin{align*}
    & max^i_Q(i_w) = \\
    & = \begin{cases} 
	V_Q(w) \ominus F_Q(w) &\text{if $Out_w=\emptyset$}\\
        \max\limits_{\forall o_{w},\forall i_{w'} \in ch_G(o_w)}(max^i_Q(i_{w'})) \ominus F_Q(w) &\text{if $Out_w \neq \emptyset$}\\
    \end{cases}
\end{align*}

 The value of $max^i_Q$ for each input in the graph can be easily calculated by propagating the bounds from $Si$ to $So$. For example, in Fig. \ref{fig:RealExample}, we start computing the maximum bound of the inputs of $Si$ (\textit{xsd:boolean}). Since $Si$ has no outputs, $max^i_Q(\textit{xsd:boolean})$ is calculated as $V_Q(Si)\ominus F_Q(Si)=410~ms-0~ms$. Then, we select all the services whose outputs match \textit{xsd:boolean}. In this case there is just one service, \textit{ML Predictor Service}. The bounds of  its inputs are now computed by subtracting out the $F_Q(\textit{ML Predictor Service})$ from the maximum bound of the inputs that this service matches. Since there is just one input matched (\textit{xsd:boolean} from $Si$) whose bound is \textit{410 ms}, we have $max^i_Q(i) = 410ms - 210ms = 200ms$ for each input $i$ of the service. In the next step, we have three services that match the new calculated inputs (\textit{Free Geoloc Service}, \textit{Premium Geoloc Service} and \textit{Transaction 
Service}). The maximum bounds of the inputs of these services are $200ms - 180ms = 20ms$, $200ms - 40ms = 160ms$ and $200ms - 130ms=70ms$ respectively. Note that, since the maximum bound of \textit{Transaction Service} is \textit{70ms}, the service \textit{Secure Payment} is out of the bounds (its output QoS is \textit{80 ms}), so it can be safely pruned.

The \textbf{third and the forth pass} analyze service equivalences and dominances in the \textit{Service Match Graph}. It is very frequent to find services from different providers that offer similar services with overlapping interfaces (inputs/outputs). In scenarios like this, it is easy to end up with large \textit{Service Match Graph} that make very hard to find optimal compositions in reasonable time. One way to reduce the complexity without losing information is to analyze the interface equivalence and dominance between services in order to \textit{combine} those that are equivalent, or replace those that are dominated in terms of the interface they provide and the QoS they offer. In a nutshell, we check three objectives to compare services: the \textit{amount} of information they need to be invoked (inputs), the \textit{amount} of information they return (outputs), and their QoS. If a set of services are equal in all objectives, they are equivalent and they can be combined into an abstract service 
with several possible implementations. If a service is equal in all objectives and at least better in one objective (it requires less information to be invoked, produces more information or has a better QoS), then the service \textit{dominates} the other service. A more detailed description of the interface and dominance optimizations is described in \cite{RodriguezMier2015}.

Note that optimizations are applied right before all semantic matches are computed in the \textit{Service Match Graph}, since the optimizations are based on the analysis of the I/O matches among services. For this reason, they cannot be applied during the calculation of the graph (this would require to precompute in advance missing relations during the graph generation, which does not provide any benefit as this is what the \textit{Service Match Graph} generation algorithm already does). On the other hand, optimizations are applied sequentially to save computation time, since the number of services in the graph decreases monotonically in each step. In order to take advantage of this, faster optimizations are applied first so that the slower optimizations in the pipeline can work with a reduced set of services.

\subsection{Hybrid algorithm}
Each service in the composition graph may have different services that match each input, thus there may exist multiple combinations of services that satisfy the composition request with the same or different QoS. The goal of the hybrid search is to extract good solutions from the composition graph, optimizing the total number of involved services in the composition and guaranteeing the optimal QoS. Thus, for each input we select just one service of the graph to match that input, until the best combination is found. The hybrid search performs a local search to extract a good solution and in the remaining time, it tries to improve the solution by running a global search.

Fig. \ref{alg:LocalSearch} shows the pseudocode of the \textbf{local search} strategy. The algorithm starts with a composition graph, the inputs of the service $Si$ marked as unresolved (the expected outputs of the request) and the service $Si$ selected to be part of the solution. An \textit{unresolved input} is an input that can be matched by many different outputs but no decision has been made yet. Using the list of the unresolved inputs to be matched, the method RANK-RESOLVERS returns a list of services that match any of the unresolved inputs. Services are ranked according to the number of unresolved inputs that match, so the service that matches more inputs is considered first to be part of the solution. Then, for each input that the selected service can match, the method CYCLE performs a forward search to check if resolving the selected input with that service leads to a cycle. For example, in Fig. \ref{fig:GraphExample}, if we select the service $K$ to match the input of $I$ after having decided to 
resolve the input of $K$ with the service $I$, we end up with an invalid composition, so $K$ is an invalid resolver for $I$ and it must be 
discarded. Once all resolvable inputs are collected in $resolved$, the method RESOLVE creates a copy of the current graph where the inputs in $unresolved$ are matched only by the selected service, i.e., any other match between any output from a different service to that input is removed from the graph. If the selected service was not already selected, then all its inputs are then marked as unresolved and a recursive call to LSBT is performed to select a new service to resolve the remaining inputs, until a solution is found. If a dead end is reached (a solution that has no services to resolve the remaining inputs without cycles) the algorithm backtracks to a previous state to try a different service (L.\ref{alg:ls:back}).

\begin{figure}
\begin{algorithmic}[1]
\Function{LOCAL-SEARCH}{$G_S=\{V,E\}$}
  \State
  \Return LSBT($G_S$, $In_{Si}$, $\{Si\}$)
\EndFunction
\State
\Function{LSBT}{$G_S$, $unresolved$, $services$}
  \If {$unresolved=\emptyset$} \Return $G_S$ \EndIf
  \State $servs \gets$ RANK-RESOLVERS($unresolved$)
  \For{\textbf{each} $w \in servs$} \label{alg:ls:back}
    \State $resolved \gets \{\}$
    \State $matched \gets Out_w \otimes unresolved$
    \For{\textbf{each} $input \in matched$}
      \If {$\neg$CYCLE($G_S,w,input$)} 
	\State $resolved \gets resolved \cup input$  
      \EndIf
    \EndFor
    \If {$resolved \neq \emptyset$}
      \State $unresolved \gets unresolved \setminus resolved$
      \If {$w \notin services$}
	\State $unresolved \gets unresolved\ \cup$ $In_w$
      \EndIf
      \State $G'_S \gets$ RESOLVE($G_S$, $w$, $resolved$)
      \State $services \gets services \cup w$
      \State $result \gets$ LSBT($G'_S$, $unresolved$, $services$)
      \If {$result \neq$ \textit{fail}} \Return $result$ \EndIf
    \EndIf
  \EndFor
  \Return \textit{fail}
\EndFunction
\end{algorithmic}
\caption{Local search algorithm to extract a composition from a graph.}
\label{alg:LocalSearch}
\end{figure}

An implementation of the \textit{CYCLE} method is provided in \ref{alg:cycle}. The algorithm performs a \textit{look-ahead} check in a breadth-first fashion to determine whether matching the selected input $i$ with an output of the service $w$ leads to a cyclic dependency. This is done by traversing only the resolved matches, i.e., inputs that are matched by just one output of a service, until the selected service $w$ is reached, proving the existence of a cycle. A more memory efficient implementation of the cycle algorithm can be done using the \textit{Tarjan's strongly connected components} algorithm \cite{tarjan1972depth}, stopping at the first strongly connected component detected. 

\begin{figure}
\begin{algorithmic}[1]
\Function{CYCLE}{$G_S=\{V,E\}, w, i_{w'}$}
  \State $W_{visited} \gets \{w'\}$
  \State $W_{new} \gets \{w'\}$
  \While{$W_{new} \neq \emptyset$}
      \State $W_{reached} \gets \{\}$
      \ForAll{$w_n \in W_{new}$}
	  \ForAll{$o_{w_n} \in Out_{w_n}$}
	    \ForAll{$i_{w'_n} \in ch_{G_S}(o_{w_n})$}
	      \If{$d^-_{G_S}(i_{w'_n}) = 1 \wedge w'_n \notin W_{visited}$}
	        \If{$w'_n = w$}
		  \Return $true$
	        \EndIf
		\State $W_{reached} \gets W_{reached} \cup w'_n$
	      \EndIf
	    \EndFor
	  \EndFor
      \EndFor
      \State $W_{new} \gets W_{reached}$
      \State $W_{visited} \gets W_{visited} \cup W_{new}$
  \EndWhile
  \State
  \Return $false$
\EndFunction
\end{algorithmic}
\caption{Näive breadth-first-search algorithm to check whether using the service $w$ to resolve the input $i_{w'}$ of a service $w'$ leads to a cycle.}
\label{alg:cycle}
\end{figure}

After the local search is used to find a good solution, the \textbf{global search} is performed in the remaining time to obtain a better solution by exhaustively exploring the space of possible solutions. 
In a nutshell, this algorithm works as follows: Given a \textit{Service Match Graph} $G_S$, with some unresolved inputs, which initially are the inputs of the service $Si$, the algorithm selects an input to be resolved and for each service candidate that can be used to resolve that input, it generates a copy of the graph $G_S$ but with the input resolved (i.e., the selected service is the only one that matches the unresolved input). The algorithm enqueues each new graph to be expanded again, and repeats the process by extracting the graph with the minimum number of services from the queue, until it eventually finds a graph with no unresolved inputs.

Fig. \ref{alg:GlobalSearch} shows the pseudocode of the global search algorithm. The algorithm starts computing the optimal QoS of the graph with the method \textit{QoS-UPDATE}. This method returns a key-value table $qos[i,q]$ where each key corresponds with an input $i$ of the graph, and each value $q$ its optimal aggregated QoS $q=V^{in}_Q(i)$. Then, the inputs of the service $Si$ of the graph are added to $I_{un}$ to mark them as unresolved (L.\ref{alg:gs:unresolved}). In order to minimize the number of possible candidates for each unresolved input, we compute and propagate a range of valid QoS values, called QoS bounds, and defined as an interval $[min,max]$. These bounds determine the range of valid accumulated QoS values of the outputs that can be used to match each of the unresolved inputs without exceeding the optimal end-to-end QoS of the final composition. The $min$ value is the optimal QoS for the input, i.e., there is no output in the graph that can match the input with a lower QoS, whereas the 
$max$ value is the maximum QoS value supported. If this bound is exceeded, the total aggregated QoS of the composition worsens. For example, in Fig. \ref{fig:RealExample}, the bounds of the input \textit{ont4:ClientID} of the service \textit{Premium Geoloc Service} are $[20ms, 160ms]$. If we exceed the min bound (\textit{20 ms}), the output QoS of the service gets worse ($> 60 ms$), which also affects the optimal QoS of the input \textit{ont1:Location}. However, as long as the max bound is not exceeded ($\leq 160 ms$), the optimal accumulated QoS of the \textit{ML Predictor Service} would not be affected.

The method \textit{COMPUTE-$V_Q$} is used to compute the value of the $V_Q$ function (Eq. \ref{eq:vq}) using the best QoS values of inputs, stored in $qos$ ($qos[i]=V^{in}_Q(i)$). A tuple $\langle  G_S, I_{un}, qos, W_{sel} \rangle$, where $G_S$ is the current graph, $I_{un}$ are the unresolved inputs of $G_S$, $qos$ is the best aggregated QoS values for each input in $G_S$ and $W_{sel}$ is the set of the selected services, defines the components of a partial solution. Each partial solution is stored in a priority queue, which is sorted by the number of services $W_{sel}$. This allows an exploration of the search space in a breadth-first fashion, so the solution with the minimum number of services is always expanded first. At each iteration, a partial solution is extracted from the queue to be refined (L.\ref{alg:gs:pop}). If the partial solution has no unresolved inputs, the solution is complete, and has the minimum number of services. If the partial solution still has some unresolved inputs, it is refined 
by selecting an unresolved input with the method \textit{SELECT}. This method selects the input to be resolved, using a \textit{minimum-remaining-values} heuristic. This heuristic selects always the input with less resolvers (services candidates) in order to minimize the branching factor. The list of services that can match the selected input with a total aggregated QoS value within the $[min, max]$ bound is calculated with the method \textit{RESOLVERS}. For each valid service, the algorithm performs a \textit{look-ahead} search to check whether using the current service to resolve the selected input leads to an unavoidable cycle. If so, the service is prematurely discarded to save computation time and space. If it does not lead to a cycle, then a copy of the graph ($G'_S$) with the selected input resolved is generated, and the input is also removed from the set of unresolved inputs. Using the optimal aggregated QoS values for the inputs of the graph, stored in $qos$, the algorithm computes the 
aggregated QoS value of the service $w$. If this value is worse than the $min$ bound (COMPUTE-$V_Q(w,qos') \succ min$), then the aggregated QoS value of some inputs and outputs of the graph may be affected. Thus, a repropagation of the QoS values for each input and output is computed again over the new graph $G'_S$ (L.\ref{alg:gs:qos}). For example, if the \textit{Business Service Info} increments its response time to \textit{40 ms}, a repropagation is required to recompute the accumulated QoS of all the services that may be affected. In this case, the \textit{Premium Geoloc Service} increments its accumulated QoS cost from \textit{60 ms} to \textit{80 ms}, as well as the optimal QoS of the \textit{ont1:Location}.

Finally, if the current service is not part of the current solution, its inputs are added to the unresolved table, and a new bound for each input is computed. The $min$ bound corresponds with the optimal value, which is stored in $qos'$. In order to compute the $max$ bound, we need to \textit{subtract} the QoS of the selected service ($F_Q(w)$) from the $max$ bound of the resolved input, using the operator $\ominus$ (L.\ref{alg:gs:maxbound}). This new partial solution is inserted in the queue to be expanded later on.

\begin{figure}
\begin{algorithmic}[1]
\Function{GLOBAL-SEARCH}{$G_S$}
  \State $qos[i,q] \gets$ QoS-UPDATE($G_S$)
  \State $max \gets $COMPUTE-$V_Q(Si,qos)$
  \State $W_{sel} \gets \{Si\}$
  \State \textit{/* $I_{un}$ is a key-value table where the keys are}
  \State \textit{unresolved inputs and the values their QoS bounds */}
  \For{$i_{Si} \in$ $In_{Si}$}
    \State $I_{un}[i_{Si}] \gets [qos[i_{Si}], max]$ \label{alg:gs:unresolved}
  \EndFor
  \State \textit{/* Queue sorted by $|W_{sel}|$*/}
  \State $queue \gets$ INSERT($\langle G_S, I_{un}, qos, W_{sel} \rangle$,$queue$)
  \While{$queue \neq \emptyset$}
      \State $\langle G_S,I_{un},qos,W_{sel} \rangle \gets$ POP($queue$) \label{alg:gs:pop}
      \If{$I_{un} = \emptyset$} \Return $G_S$ \EndIf
      \State $input \gets$ SELECT($I_{un}$)
      \State $[min,max] \gets I_{un}[input]$
      \ForAll{$w \in$ RESOLVERS($input$, $[min,max]$)}
      	\If{$\neg$CYCLE($G_S,w,input$)}
		\State $G'_S \gets$ RESOLVE($G_S,w,\{input\}$)
		\State $I'_{un} \gets$ REMOVE($i,I_{un}$)
		\State $qos' \gets qos$
		\If{COMPUTE-$V_Q(w,qos') \succ min$}  \label{alg:gs:comp}
		\State $qos' \gets$ QoS-UPDATE($G'_S$) \label{alg:gs:qos}
		\EndIf
		\If{$w \notin W_{sel}$}
		\State $W'_{sel} \gets W_{sel} \cup w$
		\State $max' \gets max \ominus F_Q(w)$ \label{alg:gs:maxbound}
		\For{$i_w \in In_w$}
			\State $min' \gets qos'[i_w]$
			\State $I'_{un}[i_w] \gets [min',max']$
		\EndFor
		\EndIf
		\EndIf
		\State $queue \gets$ INSERT($\langle G'_S, I'_{un}, qos', W'_{sel}\rangle,queue$)
   	\EndFor
  \EndWhile
  \Return \textit{fail}
\EndFunction
\end{algorithmic}
\caption{Global search algorithm to extract the optimal composition.}
\label{alg:GlobalSearch}
\end{figure}

\section{Evaluation}
\label{sec:Evaluation}
In order to evaluate the performance of the proposed approach, we conducted two different experiments. In the first experiment, we evaluated the approach using the datasets of the Web Service Challenge 2009-2010 \cite{kona2009wsc}. The goal of this first experiment was to evaluate the peformance and scalability of the proposed approach on large-scale service repositories. In the second experiment, we tested the algorithm with five random datasets in order to better analyze the differences of the performance between the local and the global search. All tests were executed with a time limit of 5 min. Solutions produced by our algorithm are represented as \textit{Service Composition Graphs} (no BPEL was generated).

\begin{table}[htbp]
  \centering
  \caption{Validation with the WSC 2009-2010}
  \def\arraystretch{1.2}
  \begin{tabular}{|l|l|c|c|c|c|c|}
  \cline{3-7}
  \multicolumn{ 2}{c|}{} & \textbf{D-01} & \textbf{D-02} & \textbf{D-03} & \textbf{D-04} & \textbf{D-05} \\ \hline
  \multicolumn{ 2}{|c|}{\textbf{\#Services in the dataset}} & 572 & 4,129 & 8,138 & 8,301 & 15,211 \\ \hline \hline
  \multicolumn{ 7}{|c|}{\textbf{Validation with Response Time}} \\ \hline
  \multicolumn{ 2}{|c|}{\textbf{Optimal Response Time (ms)}} & 500 & 1,690 & 760 & 1,470 & 4,070 \\ \hline
  \multicolumn{ 2}{|l|}{\#Graph services} & 81 & 141 & 154 & 331 & 238 \\ \hline
  \multicolumn{ 2}{|l|}{\#Graph services (opt)} & 21 & 57 & 15 & 160 & 126 \\ \hline
  \multicolumn{ 1}{|l|}{\multirow{2}*{\textbf{Local Search}}} & \#Services & 5 & 20 & 10 & 40 & 32 \\ \cline{ 2- 7}
  \multicolumn{ 1}{|l|}{} & Time (s) & 0.613 & 0.988 & 2.608 & 7.767 & 2.920 \\ \hline
  \multicolumn{ 1}{|l|}{\multirow{2}*{\textbf{Global Search}}} & \#Services & 5 & 20 & 10 & - & 32 \\ \cline{ 2- 7}
  \multicolumn{ 1}{|l|}{} & Time (s) & 0.617 & 1.580 & 2.613 & - & 24.971 \\ \hline \hline
  \multicolumn{ 7}{|c|}{\textbf{Validation with Throughput}} \\ \hline
  \multicolumn{ 2}{|c|}{\textbf{Optimal Throughput (inv/s)}} & 15,000 & 6,000 & 4,000 & 4,000 & 4,000 \\ \hline
  \multicolumn{ 2}{|l|}{\#Graph services} & 81 & 141 & 154 & 331 & 238 \\ \hline
  \multicolumn{ 2}{|l|}{\#Graph services (opt)} & 10 & 43 & 90 & 156 & 69 \\ \hline
  \multicolumn{ 1}{|l|}{\multirow{2}*{\textbf{Local Search}}} & \#Services & 5 & 20 & 15 & 62 & 31 \\ \cline{ 2- 7}
  \multicolumn{ 1}{|l|}{} & Time (s) & 0.343 & 1.173 & 1.933 & 8.571 & 2.562 \\ \hline
  \multicolumn{ 1}{|l|}{\multirow{2}*{\textbf{Global Search}}} & \#Services & 5 & 20 & 10 & - & 30 \\ \cline{ 2- 7}
  \multicolumn{ 1}{|l|}{} & Time (s) & 0.345 & 1.246 & 2.085 & - & 119.322 \\ \hline
  \end{tabular}
  \label{table:results}
\end{table}

\subsection{Web Service Challenge 2009-2010 datasets}

The datasets of the Web Service Challenge 2009-2010 range from 572 to 15,211 services with two different QoS properties: response time and throughput. 
Table \ref{table:results} shows the results obtained for each dataset and for each QoS property. The response time is the average time (measured in milliseconds) that a service takes to respond to a request. The throughput, as defined in the WSC, is the average ratio of invocations per second supported by a service. 

Row \textit{\#Graph services} shows the number of services of the composition graph and \textit{\#Graph services (opt)} the number of services after applying the graph optimizations. As can be seen, the optimizations reduce, on average, by 64\% the number of services in the initial composition graph. This indicates that equivalence and dominance analysis of the QoS and the functionality of services is a powerful technique to reduce the search space in large scale problems. Rows \textit{Local search} and \textit{Global search} show the number of services of the solution obtained with each respective method as well as the total amount of time spent in the search. The global search found the best solution for each dataset and for each QoS property, except for the dataset 04, where the composition with the minimum number of services could not be found due to combinatorial explosion. However, in those cases, the local search strategy is able to find an alternative solution very fast. Note also that, in many cases,
 the local search obtains the best solution (comparing it with the global search) except for the throughput in datasets 03 and 05. 

\begin{table}[htbp]
\caption{Comparison with the top 3 WSC 2010}
\def\arraystretch{1.2}
\begin{tabular}{|l|l|c|c|c|c|}
\cline{3-6}
\multicolumn{ 2}{c|}{} & \textbf{R.Time} & \textbf{Through.} & \textbf{Min. Serv.} & \textbf{Time (ms)} \\ \hline
\multicolumn{ 1}{|c|}{\multirow{4}*{\textbf{D-01}}} & CAS \cite{Jiang2010a} & 500 & 15,000 & \cellcolor[gray]{0.8}\textbf{5} & 78 \\ \cline{ 2- 6}
\multicolumn{ 1}{|l|}{} & RUG \cite{Aiello2009} & 500 & 15,000 & 10 & 188 \\ \cline{ 2- 6}
\multicolumn{ 1}{|l|}{} & Tsinghua \cite{Yan2009} & 500 & 15,000 & 9 & 109 \\ \cline{ 2- 6}
\multicolumn{ 1}{|l|}{} & \textbf{Our approach} & 500 & 15,000 & \cellcolor[gray]{0.8}\textbf{5} & 956 \\ \hline \hline
\multicolumn{ 1}{|c|}{\multirow{4}*{\textbf{D-02}}} & CAS  \cite{Jiang2010a} & 1,690 & 6,000 & \cellcolor[gray]{0.8}\textbf{20} & 94 \\ \cline{ 2- 6}
\multicolumn{ 1}{|l|}{} & RUG \cite{Aiello2009} & 1,690 & 6,000 & 40 & 234 \\ \cline{ 2- 6}
\multicolumn{ 1}{|l|}{} & Tsinghua \cite{Yan2009} & 1,690 & 6,000 & 36 & 140 \\ \cline{ 2- 6}
\multicolumn{ 1}{|l|}{} & \textbf{Our approach} & 1,690 & 6,000 & \cellcolor[gray]{0.8}\textbf{20} & 2,171 \\ \hline \hline
\multicolumn{ 1}{|c|}{\multirow{4}*{\textbf{D-03}}} & CAS  \cite{Jiang2010a} & 760 & 4,000 & \cellcolor[gray]{0.8}\textbf{10} & 78 \\ \cline{ 2- 6}
\multicolumn{ 1}{|l|}{} & RUG \cite{Aiello2009} & 760 & 4,000 & 11 & 234 \\ \cline{ 2- 6}
\multicolumn{ 1}{|l|}{} & Tsinghua \cite{Yan2009} & 760 & 4,000 & 18 & 125 \\ \cline{ 2- 6}
\multicolumn{ 1}{|l|}{} & \textbf{Our approach} & 760 & 4,000 & \cellcolor[gray]{0.8}\textbf{10} & 4,693 \\ \hline \hline
\multicolumn{ 1}{|c|}{\multirow{4}*{\textbf{D-04}}} & CAS  \cite{Jiang2010a} & 1,470 & 4,000 & 73 & 156 \\ \cline{ 2- 6}
\multicolumn{ 1}{|l|}{} & RUG \cite{Aiello2009} & 1470 & 4,000 & 133 & 390 \\ \cline{ 2- 6}
\multicolumn{ 1}{|l|}{} & Tsinghua \cite{Yan2009} & 1,470 & 4,000 & 133 & 188 \\ \cline{ 2- 6}
\multicolumn{ 1}{|l|}{} & \textbf{Our approach} & 1,470 & 4,000 & \cellcolor[gray]{0.8}\textbf{40} & 16,338 \\ \hline \hline
\multicolumn{ 1}{|c|}{\multirow{4}*{\textbf{D-05}}} & CAS  \cite{Jiang2010a} & 4,070 & 4,000 & 32 & 63 \\ \cline{ 2- 6}
\multicolumn{ 1}{|l|}{} & RUG \cite{Aiello2009} & 4,070 & 4,000 & 4,772 & 907 \\ \cline{ 2- 6}
\multicolumn{ 1}{|l|}{} & Tsinghua \cite{Yan2009} & 4,070 & 4,000 & 4,772 & 531 \\ \cline{ 2- 6}
\multicolumn{ 1}{|l|}{} & \textbf{Our approach} & 4,070 & 4,000 & \cellcolor[gray]{0.8}\textbf{30} & 122,242 \\ \hline
\end{tabular}
\label{table:comparison}
\end{table}

We have compared our approach with the top-3 of the Web Service Challenge 2010 \cite{weise2010web}. Table \ref{table:comparison} shows this comparison following the same format and the same rules of the Web Service Challenge. The format, rules and other details of the challenge are described in \cite{weise2010web}. Third and forth columns show the response time and the throughput obtained for each dataset. Note that, since all these algorithms minimize a single QoS, these values are computed by executing the algorithm twice, one for each QoS. Unfortunately, the results provided by the WSC organization in  \cite{weise2010web} show only the minimum number of services for both executions (fifth column). Thus, the number of services obtained for both the response time and throughput is unknown, which makes it hard to compare with our results. Even so, using the same evaluation criteria, our approach obtains the optimal QoS for the response time and the throughput, and also improves the number of services in D-04 
(40 vs 73) and D-05 (30 vs 32) with respect to the solutions obtained by the winner of the challenge (the minimum number of services obtained for each dataset is highlighted). The last column shows the total execution time of each algorithm. The total time includes the time spent to obtain the solution for the response time and for the throughput. 

Our approach takes, in general, more time to obtain a solution. 
However, it should be noted that we show the best results achieved by the hybrid approach, i.e., if the global search improves the solution of the local search, we show that solution along with the time taken by the global search. Anyway, the local search always provide a first good solution very fast. For example, as can be seen in Table \ref{table:results}, the optimal solution for D-05 has 30 services and has been obtained in 119.322 s, but the local search obtained a solution with 31 services in 2.56 s, still better than the solution with 32 services obtained by \cite{Jiang2010a} (Table \ref{table:comparison}). Moreover, it should also be noted that the problem of finding the optimal composition with minimum number of services and optimal QoS is much harder than just optimizing the QoS objective function, which is the problem solved by the participants of the WSC 2010. Although the problem is intractable and requires exponential time, it can be optimally solved for many particular instances in a 
reasonable amount of time using adequate optimizations even in large datasets as shown in Tables \ref{table:results} and \ref{table:results2}. This is one of the main reasons why a combination of a local and global search can achieve good results in a wide variety of situations, in contrast with pure greedy strategies or with pure global optimization algorithms.

We also compare the results obtained with Chen et al. \cite{Chen2012}, who offer a detailed analysis of their results. This comparison is shown in Table \ref{table:comparison2}. Solutions are compared according to their QoS and number of services. A solution is better if 1) its overall QoS is better or 2) has the same QoS but less services. The results show that our algorithm always gets same or better results. Concretely, it finds solutions with optimal QoS and less services in D-01, D-02, D-04 and D-05 (response time), and D-03 (throughput). It also finds a solution with a better QoS (4000 inv/s vs 2000 inv/s) in D-04 (throughput).

\begin{table}[htbp]
\caption{Detailed comparison with \cite{Chen2012}}
\def\arraystretch{1.2}
\begin{tabular}{l l|c|c|c|c|c|}
\cline{3-7}
 &  & \textbf{D-01} & \textbf{D-02} & \textbf{D-03} & \textbf{D-04} & \textbf{D-05} \\ \hline
\multicolumn{ 1}{|c|}{\multirow{2}*{Chen et al.}} & R. Time & 500 & 1,690 & 760 & 1,470 & 4,070 \\ \cline{ 2- 7}
\multicolumn{ 1}{|l|}{} & Services & 8 & 21 & \cellcolor[gray]{0.8}\textbf{10} & 42 & 33 \\ \hline
\multicolumn{ 1}{|c|}{\multirow{2}*{Our approach}} & R. Time & 500 & 1,690 & 760 & 1,470 & 4,070 \\ \cline{ 2- 7}
\multicolumn{ 1}{|l|}{} & Services & \cellcolor[gray]{0.8}\textbf{5} & \cellcolor[gray]{0.8}\textbf{20} & \cellcolor[gray]{0.8}\textbf{10} & \cellcolor[gray]{0.8}\textbf{40} & \cellcolor[gray]{0.8}\textbf{32} \\ \hline \hline
\multicolumn{ 1}{|c|}{\multirow{2}*{Chen et al.}} & Throughput & 15,000 & 6,000 & 4,000 & 2,000 & 4,000 \\ \cline{ 2- 7}
\multicolumn{ 1}{|l|}{} & Services & \cellcolor[gray]{0.8}\textbf{5} & \cellcolor[gray]{0.8}\textbf{20} & 21 & 40 & \cellcolor[gray]{0.8}\textbf{30} \\ \hline
\multicolumn{ 1}{|c|}{\multirow{2}*{Our approach}} & Throughput & 15,000 & 6,000 & 4,000 & \cellcolor[gray]{0.8}\textbf{4,000} & 4,000 \\ \cline{ 2- 7}
\multicolumn{ 1}{|l|}{} & Services & \cellcolor[gray]{0.8}\textbf{5} & \cellcolor[gray]{0.8}\textbf{20} & \cellcolor[gray]{0.8}\textbf{10} & 62 & \cellcolor[gray]{0.8}\textbf{30} \\ \hline
\end{tabular}
\label{table:comparison2}
\end{table}

\subsection{Randomly generated datasets}
Although the global search is able to obtain solutions with a lower number of services, a first look at the results with the WSC dataset might suggest that the difference of both strategies is not very significant, as most of the obtained solutions have the same number of services. However, this may be due to a bias in the repository, since all the datasets of the WSC are generated using the same random model. In order to better evaluate and characterize the performance of the hybrid algorithm, we generated a new set of five random datasets that range from 1,000 to 9,000 services. These datasets are available at \url{https://wiki.citius.usc.es/inv:downloadable_results:ws-random-qos}. Table \ref{table:results2} shows the solutions obtained.

\begin{table}[htbp]
\centering
\caption{Validation with random datasets}
\def\arraystretch{1.2}
\begin{tabular}{|l|l|c|c|c|c|c|}
\cline{3-7}
\multicolumn{ 2}{c|}{} & \textbf{R-01} & \textbf{R-02} & \textbf{R-03} & \textbf{R-04} & \textbf{R-05} \\ \hline
\multicolumn{ 2}{|c|}{\textbf{\#Services in the dataset}} & 1,000 & 3,000 & 5,000 & 7,000 & 9,000 \\ \hline \hline
\multicolumn{ 7}{|c|}{\textbf{Validation with Response Time}} \\ \hline
\multicolumn{ 2}{|c|}{\textbf{Optimal Response Time (ms)}} & 1,430 & 975 & 805 & 1,225 & 1,420 \\ \hline
\multicolumn{ 2}{|l|}{\#Graph Services} & 54 & 168 & 285 & 383 & 499 \\ \hline
\multicolumn{ 2}{|l|}{\#Graph Services (opt)} & 22 & 50 & 54 & 56 & 99 \\ \hline
\multicolumn{ 1}{|c|}{\multirow{2}*{\textbf{Local Search}}} & \#Services & 7 & 18 & 20 & 15 & 19 \\ \cline{ 2- 7}
\multicolumn{ 1}{|l|}{} & Time (s) & 0.183 & 0.403 & 0.422 & 0.515 & 0.641 \\ \hline
\multicolumn{ 1}{|c|}{\multirow{2}*{\textbf{Global Search}}} & \#Services & 7 & 14 & 15 & 15 & 16 \\ \cline{ 2- 7}
\multicolumn{ 1}{|l|}{} & Time (s) & 0.243 & 0.767 & 4.088 & 0.740 & 3.131 \\ \hline \hline
\multicolumn{ 7}{|c|}{\textbf{Validation with Throughput}} \\ \hline
\multicolumn{ 2}{|c|}{\textbf{Optimal Throughput (inv/s)}} & 1,000 & 2,500 & 1,500 & 2,000 & 2,500 \\ \hline
\multicolumn{ 2}{|l|}{\#Graph Services} & 54 & 168 & 285 & 383 & 499 \\ \hline
\multicolumn{ 2}{|l|}{\#Graph Services (opt)} & 19 & 46 & 133 & 116 & 103 \\ \hline
\multicolumn{ 1}{|c|}{\multirow{2}*{\textbf{Local Search}}} & \#Services & 7 & 17 & 24 & 19 & 23 \\ \cline{ 2- 7}
\multicolumn{ 1}{|l|}{} & Time (s) & 0.072 & 0.143 & 0.606 & 0.732 & 0.450 \\ \hline
\multicolumn{ 1}{|c|}{\multirow{2}*{\textbf{Global Search}}} & \#Services & 7 & 12 & 12 & 15 & 16 \\ \cline{ 2- 7}
\multicolumn{ 1}{|l|}{} & Time (s) & 0.155 & 0.310 & 2.479 & 1.485 & 1.714 \\ \hline
\end{tabular}
\label{table:results2}
\end{table}

We found that in these datasets, the solutions obtained with the global search strategy are, on average, $\approx16\%$ smaller than the ones obtained with the local search, whereas the differences in seach time are less pronounced than in the previous experiment. These findings suggest that the performance of each strategy highly depends on the underlying structure of the service repository, which is mostly determined by the number of services and the existing matching relations. 

In order to test whether these differences are statistically significant or not, we conducted a nonparametric test using the \textit{binomial sign test} for two dependent samples with a total of 20 datasets (5 WSC w/response time + 5 WSC w/throughput + 5 Random w/response time + 5 Random w/throughput). The null hypothesis was rejected with $\textit{p-value}\approx0.01$ \cite{rodriguez-fdez2015stac}, meaning that both strategies (local and global search) find significantly different solutions. Thus, a hybrid strategy can perform better in many different scenarios, since it achieves a good tradeoff between quality and execution time.

This evaluation shows that, on one hand, the combination of local and global optimization is a general and powerful technique to extract optimal compositions in diverse scenarios, as it brings the best of both worlds. This is specially important when only a little or nothing is known concerning the structure of the underlying repository of services. On the other hand, the results obtained with the Web Service Challenge 2009-2010 show that the hybrid strategy performs better than the state-of-the-art, obtaining solutions with less services and optimal QoS.

\section{Conclusions}
\label{sec:Conclusions}
In this paper we have presented a hybrid algorithm to automatically build semantic input-output based compositions minimizing the total number of services while guaranteeing the optimal QoS. The proposed approach combines a set of graph optimizations and a local-global search to extract the optimal composition from the graph. Results obtained with the Web Service Challenge 2009-2010 datasets show that the combination of graph optimizations with a local-global search strategy performs better than the state-of-the-art, as it obtained solutions with less services and optimal QoS. Moreover, the evaluation with a set of randomly generated datasets shows that the hybrid strategy is well suited to perform compositions in diverse scenarios, as it can achieve a good tradeoff between quality and execution time.

%%%%%%%%%%%%%%%%%%% APPENDIX %%%%%%%%%%%%%%%%%%%%%%

\appendices

\section{Computational Complexity}
\label{appendix:nphard}
The calculation of the optimal QoS can be computed in polynomial time for a given \textit{Service Match Graph} using classical shortest path algorithms such as Dijkstra or Bellman-Ford. But, as stated in the introduction, there can exist multiple solutions with the same global QoS but different number of services. Thus, in many scenarios, optimizing the QoS objective function is not enough to provide the best possible answer. However, it turns out that optimizing the number of services of a composition is an intractable problem. The next theorem proves that the Service Minimization Problem (SMP) is a NP-Hard combinatorial optimization problem.

\begin{theorem*}
 Finding the minimum number of services whose outputs match a given set of unresolved (unmatched) concepts is a NP-Hard combinatorial optimization problem.
\end{theorem*}

\begin{proof}
We will show that the Service Minimization Problem (SMP) is NP-Hard by proving that the optimization version of the Set Cover Problem (SCP), a well-known NP-Hard problem, is polynomial-time \textit{Karp} reducible to SMP $SCP \leq_P SMP$. The optimization version of the SCP problem is defined as follows: given a set of elements $U=\{u_1,\dots,u_m\}$ and a set $S$ of subsets of $U$, find the smallest set (cover) $C \subseteq S$ of subsets of $S$ whose union is $U$. The decision version of this problem, stated as that of deciding whether exists a cover $C_{SCP}$ of size $k$ or less ($|C_{SCP}| \leq k$), is NP-Complete. We will also consider the simplest form of the SMP that can be contained in a \textit{Service Match Graph}, which is defined as follows: given a service $w_U$ and a set of candidate services $W_S=\{w_1,\dots,w_n\}$ such that $O_{w_1} \otimes I_{w_U} \neq \emptyset$ $\wedge \dots \wedge$ $O_{w_n} \otimes I_{w_U} \neq \emptyset$, select the smallest subset of services from $W_S$ such that 
the union of the outputs of the services from $W_S$, $O_{W_S}$, satisfies $O_{W_S} \otimes I_{w_U} = I_{w_U}$, i.e., the outputs of the services contained in $W_S$ match all the inputs of $w_U$. As in the SCP, the decision version of this optimization problem is defined as that of deciding whether exists a subset of candidate services $C_{SMP}$ of size $k$ or less ($|C_{SMP}| \leq k$) such that the union of the outputs of the services in $C_{SMP}$ match all the inputs of $w_U$.

In order to prove that the SMP optimization problem is NP-Hard, we need to demonstrate that its corresponding decision problem is NP-Complete. We will therefore reduce the SCP problem by means of a function $\varphi$ that transforms any arbitrary instance of the SCP into an instance of the SMP in polynomial time. We have to prove that 1) $\varphi(U,S)$ is a SMP problem; 2) $\varphi$ runs in polynomial time; and 3) there is a set covering of $\varphi(U,S)$ of size $k$ or less if and only if there is a set covering of $U$ in $S$ of size $k$ or less.

Given a pair $(U,S)$, we define $\varphi(U,S)=(w_U,W_S)$ such that:

\begin{itemize}
 \item $w_U = \{I_{w_U} = U =\{u_1,\dots,u_n\}, \emptyset \}$, where $u_i$ is the $i$th unresolved input of $w_U$.
 \item $\forall s_i = \{u_{i_1},\dots,u_{i_n}\} \in S$, $\exists w_i \in W_S$ such that $w_i=\{\emptyset, O_{w_i}\}$ and $O_{w_i} \otimes I_{w_U} = s_i$
\end{itemize}

By this definition, the $\varphi(U,S)$ maps each element $u \in U$ to an input of the service $w_U$. Each subset $s_i \in S$ is also mapped to a service whose outputs match exactly the inputs of $w_U$ that correspond with the elements of $s_i$. This mapping can be computed by adding a match from an arbitrary output of each service $w_i \in W_S$ to each input $u_i \in s_i$, which clearly runs in linear time in the size of $U$. Moreover, $\varphi(U,S)$ is a Service Minimization Problem according to its definition.

Now suppose there is a set covering $|C| \leq k, C \subseteq S$ of $U$. Thus, $\forall u \in U, \exists c_i \in C$ such that $u \in c$. From the services $(w_U, W_S)$ constructed from $(U,S)$ by $\varphi(U,S)$, there exists $w_i \in W_S$ such that $O_{w_i} \otimes I_{w_U} = c_i \subseteq I_{w_U}$, and so $\bigcup_i (O_{w_i} \otimes I_{w_U}) = I_{w_U} = C$, i.e., the outputs of the services from the set $W_S$ of size $k$ or less represent a cover of the Service Minimization Problem $\varphi(U,S)$. %Thus, it follows that the Service Minimization Problem is NP-Hard.

\end{proof}

\section{Algorithm Analysis and Discussion}
\label{appendix:complexity}

The proposed approach consists of a hybrid algorithm that optimizes both the global QoS and selects the composition with the minimum number of services that preserves the optimal QoS. As demonstrated in Appendix \ref{appendix:nphard}, the problem of minimizing the number of services is NP-Hard. Thus, under the $P \neq NP$ assumption, there is no polynomial time algorithm that can exactly solve this optimization problem. However, although it is in general intractable, in practice many instances of the problem, as shown in the evaluation section, can be optimally solved in reasonable time. In those situations, it may be preferable to provide optimal solutions instead of just sub-optimal ones. Our approach takes advantage of a hybrid strategy that combines a local search and a global search plus the use of preprocessing optimizations and search optimizations (minimum-remaining-values heuristic, cycle detection, QoS bounds propagation) in order to achieve a good trade-off between optimality of the solution and 
computation time. Here we analyze the complexity of the proposed techniques.

\subsection{Cycle detection}
The cycle detection is implemented as a Look-Ahead strategy, that traverses all the resolved matches, starting from the current service (the one selected to resolve a new unresolved input), until no more services are reachable. This strategy seeks to discover whether the current service is a valid candidate or not by checking if it can lead to a dependency cycle, so it can be prematurely discarded. The cycle detection algorithm takes $O(|V| + |E|)$, since every service, input, output and match between inputs and outputs have to be traversed in worst-case.

\subsection{QoS Update}
The QoS update method calculates the optimal end-to-end QoS through the graph. This method is also used to recalculate optimal QoS bounds whenever a local QoS bound is excedeed. This problem can be modeled as a shortest path problem with generalized costs for QoS (as shown in Section 4.3) and solved using Dijkstra's algorithm. The worst-case time complexity of this algorithm is as follows: given a  \textit{Service Match Graph} $G_S=(V,E)$, where $W_R \subset V$ is the set of services in the graph, there are at most $|W_R|$ calls to $POP$ method to extract the lowest scored service from the queue. Since the queue is implemented as a binary heap, the \textit{POP} and \textit{INSERT} methods have a time of $O(log(n))$, where $n$ is the size of the queue. Thus, in the worst case, the running time is $O(|W_R|\cdot log(|W_R|))$, plus the (at most) $|E|$ updates of neighbor services that are reinserted into the queue. Therefore, the overall time is $O( (|E| + |W_R|) \cdot log (W_R))$.

\subsection{Local search}
This method performs a heuristically guided local search to minimize the number of services of the optimal end-to-end QoS composition. At each step, it selects the most promising candidate by selecting the one with fewer inputs that matches the largest number of unresolved inputs. If the algorithm gets stuck at some point, i.e., it reaches a point where no service can be selected without leading to a cyclic dependency, it backtracks to try the next most promising candidate service. The algorithm calls \textit{RANK-RESOLVERS} to rank the candidates according to the number of unresolved inputs that each candidate can match and, in case of draw the service with less inputs is preferred. The sorting of services takes $O(n\cdot log(n))$ using merge sort, where $n$ is the number of services. Each time a service is selected, the method \textit{RESOLVE} creates an updated copy of the graph in $O(|V| + |E|)$. 

Assumming non-cyclic dependencies in the \textit{Service Match Graph}, in the worst case the algorithm have to select all the services from the graph until no unresolved inputs are left. Thus, in the first step $t_{|W_R|}$ the algorithm ranks all the $|W_R|$ services in $O(|W_R| \cdot log(|W_R|))$, selects the first one and generates a new copy of the graph in $O(|V|+|E|)$. The running time of this step is $O(|W_R| \cdot log(|W_R|) + O(|V| + |E|) = O(|W_R| \cdot log(|W_R|))$. In the next step $t_{|W_R|-1}$, the algorithm ranks $|W_R|-1$ services, selects the best one, creates a copy of the graph and so on. Therefore, the asymptotic upper bound of the running time of $t_{|W_R|} + t_{|W_R|-1} + \dots + t_1$ is $O(|W_R| \cdot log(|W_R|))$.

In the absence of the assumption of non-cyclic dependencies, the asymptotic upper bound analysis shows that the time complexity grows exponentially with the depth of the search, since in the worst-case the algorithm fails (backtracks) at each step until the last combination of services is explored. However, in practice, this upper bound seems far from the average-case. As shown in the evaluation (Section 6), the growth of the time with respect to the size of the graph is closer to the best-case scenario, since an exponential number of backtracks due to cylic dependencies is extremely rare. In any case, the algorithm can be easily adapted to perform better in the worst-case scenario, for example by limiting the number of candidates to the top-\textit{K} best services for each unresolved input.

\subsection{Global search}
The aim of the global search algorithm is to perform an exhaustive search to find the minimum combination of services that satisfy the composition request with optimal QoS. The algorithm explores every possible valid combination of services in a breadth-first fashion by resolving one input at a time. For each unresolved input with $k > 1$ candidates, new $k$ different states are created by calling the \textit{RESOLVE} method and pushed to the queue for further expansion. In order to calculate an asymptotic upper bound for the time complexity, we can compute the number of combinations of services that the algorithm needs to extract from the queue in the worst-case. To this end, we first count the maximum number of combinations (solutions) that we can generate for a simple graph with fixed size and then we generalize the problem for a graph of any size.

Left graph from Figure \ref{fig:flattening} shows an example of a \textit{Service Match Graph} with 4 services (excluding $Si$ and $So$). As can be seen, $Si$ requires two inputs, 1 and 2. On the other hand, the outputs of $A$ and $B$ match the input $1$ whereas the outputs of services $C$ and $D$ match the input $2$. Therefore, in order to match both inputs, we can select services $A$ and $C$, $A$ and $D$, $B$ and $C$ or $B$ and $D$ ($2\times 2$ combinations). By computing all possible combinations, we can reduce the graph from the left, where $Si$ has two inputs, to the graph from the right, where $Si$ has just one input. %\footnote{This problem is similar to converting AND nodes into OR nodes in an AND/OR graph.}.

\begin{figure}
 \centering
 \includegraphics[width=\columnwidth]{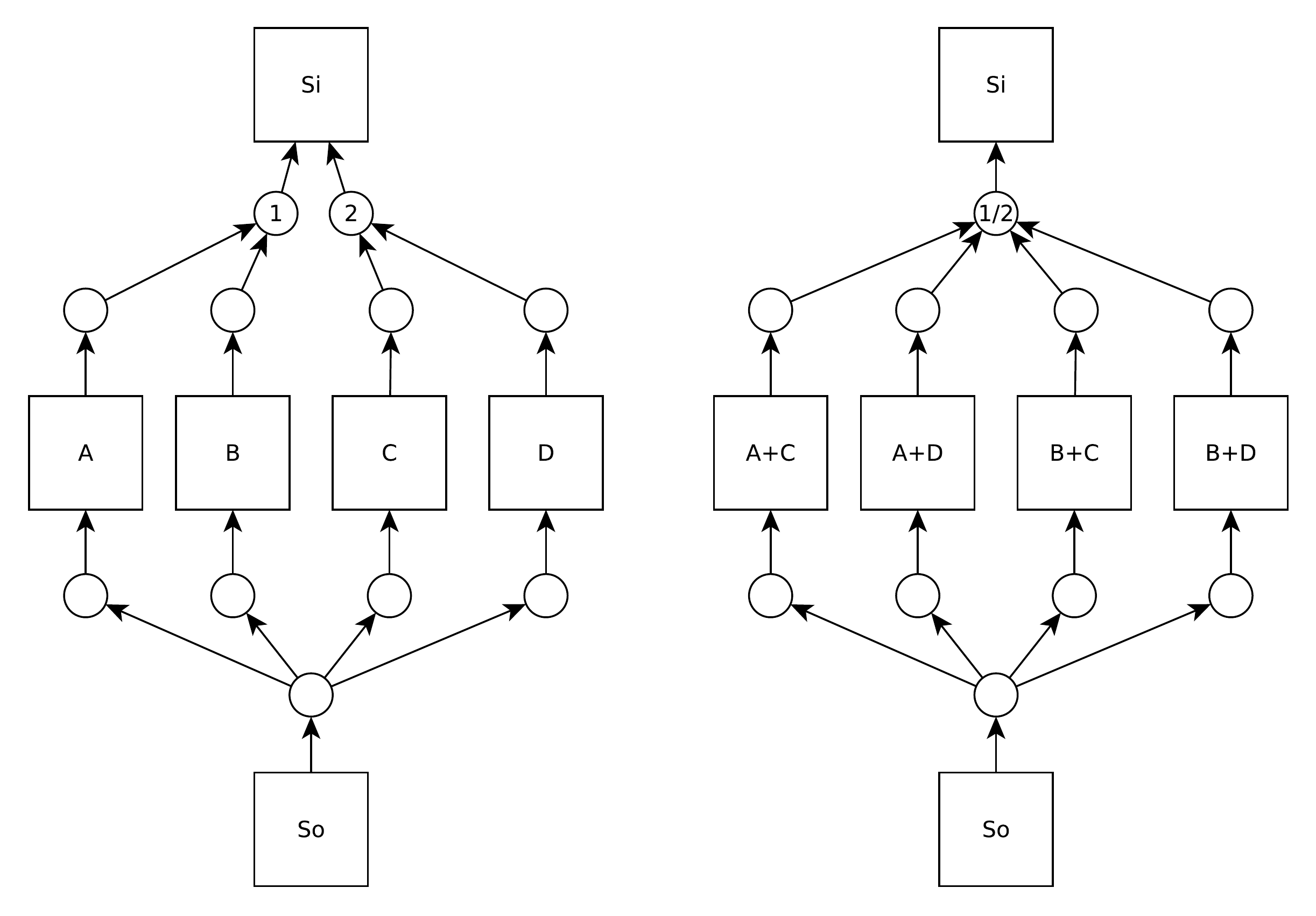}
 \caption{Reduction of the left graph into the right graph by computing all possible combinations of services}
 \label{fig:flattening}
\end{figure}

In general, given a service $w$ with $|I_w|=k$ inputs and $c_1,c_2,\dots,c_k$ set of candidate services for each input, there are $\prod_i |c_i|$ combinations of services, i.e., we can replace the $k$ inputs with $k$ sets of candidate services by one input with $\prod_i |c_i|$ candidates. Since each service can have in turn some inputs with other candidates, we can recursively replace each service with all the possible combinations of services that can be generated. This process leads to a flattening of the graph until there is just one level with all the possible combinations of services (compositions) that can be generated for a given \textit{Service Match Graph}. Thus, the problem of counting the number of possible solutions in the worst-case can be reduced to the following: given a \textit{Service Match Graph} with $|W_R|$ services, what is the maximum of products of partitions
of $W_R$? More formally, given a set 
$S$ ($|S| \geq 1$), choose $n$ partitions $c_1,c_2,\dots, c_n$ such that $\sum_i |c_i| = |S|$ and $\prod_i |c_i|$ is maximized.
For example, given 11 services, we can take 3 groups of 3 services and one with the remaining 2 services, so the product of the partition is $3^3 \cdot 2 = 54$, which is the maximum. Finding an upper bound for this value will gives us an upper bound for the maximum number of compositions that can be enumerated in the worst-case, i.e., for the most complex \textit{Service Match Graph} that can be generated with $|W_R|$ services. It can be proved that, for any set of size $n$, the maximum can be obtained by partitioning the set into groups of 2 and 3 elements, with no more than 2 groups of 2 elements. From this it follows that the maximum product is bounded by $3^{n/3}$, so we can conclude that $O(3^{n/3})$ is a tight asymptotic upper bound on the running time in the worst-case.

However, it should be noted that although the calculation of an optimal solution for the problem in the worst-case requires exponential time with the size of the graph, in practice, the number of services for a particular request is usually orders of magnitude lower that the number of available services in the dataset (see Table 2 and 4). In addition to this, the optimizations introduced in Section 5.3 plus the global QoS bound propagation, the minimum-remaining-values heuristic and cycle detection used in the global search are aimed to reduce further the size of the explored search space by decreasing the number of analyzed services.

\section{Differences with previous work}
\label{appendix:differences}
In \cite{RodriguezMier2015} we presented an integrated approach for discovery and composition of semantic Web services. However, the framework does not include any of the novelties that are presented in this approach. Our previous work presents an integrated framework for automatic I/O driven discovery and composition of semantic Web services and analyzes the impact of the discovery in the whole process, but with no QoS support. In contrast, in this work we present a hybrid composition algorithm that optimizes both QoS and the number of services, which is a different and a harder problem. The main differences are:

 \begin{itemize}
  \item The \textit{Service Model} has been extended to give support for QoS properties.
  \item The computation of the \textit{Service Match Graph} for this problem is different. In this work, all the semantic matches between all the services in the graph are computed in order to be able to guarantee an optimal end-to-end QoS. However, in \cite{RodriguezMier2015}, the \textit{Service Match Graph} contains only the matches from the outputs of previous layers to the inputs of subsequent layers, i.e., the inputs of a service that appears in the $i$th layer can be matched only by the outputs of services that are in any $j$th layer where $j \in [0, i-1]$. This condition is enough to find the smallest composition (in terms of number of services and length of the composition) but it is not enough to guarantee the optimal QoS since there are missing relations that can be part of the optimal solution.
  \item \textit{Service Match Graph} optimizations have been extended to take into account QoS. Also, a new step in the optimization pipeline has been included to prune suboptimal QoS services (i.e., services that cannot be part of the optimal solution).
  \item The proposed composition algorithm is completely different. The algorithm from \cite{RodriguezMier2015} is focused on the minimization of Web services using an A* algorithm with admissible state-space pruning. However, this technique is not enough to cope with the complexity of this new problem at large scale. Thus, we developed a new algorithm which consists of a hybrid strategy to optimize both global end-to-end QoS and the number of services, which is a different and also a harder problem.
 \end{itemize}

\section*{Acknowledgment}
This work was supported by the Spanish Ministry of Economy and Competitiveness (MEC) under grant TIN2014-56633-C3-1-R and the Galician Ministry of Education under the project CN2012/151. Pablo Rodríguez-Mier is supported by an FPU Grant from the MEC (ref. AP2010-1078).

\bibliographystyle{IEEEtran}
\bibliography{bibliography}

\end{document}